\documentclass[letterpaper, 10 pt, conference]{ieeeconf}

\usepackage{balance} 
\usepackage{comment}
\usepackage{color}
\usepackage{amsmath,amsfonts}
\usepackage{algorithm}
\usepackage{algpseudocode}
\usepackage{graphicx}
\usepackage{mathtools}
\usepackage[caption=false,font=normalsize,labelfont=sf,textfont=sf]{subfig}
\usepackage{upgreek}
\usepackage{algorithmicx}
\usepackage{array}
\usepackage{float}
\usepackage{cite}
\usepackage{textcomp}
\usepackage{stfloats}
\usepackage{soul}
\usepackage{url}
\usepackage{units}
\usepackage{verbatim}
\usepackage{xr-hyper}
\usepackage{hyperref}
\usepackage{pdfpages}

\makeatletter
\def\thickhline{%
  \noalign{\ifnum0=`}\fi\hrule \@height \thickarrayrulewidth \futurelet
   \reserved@a\@xthickhline}
\def\@xthickhline{\ifx\reserved@a\thickhline
               \vskip\doublerulesep
               \vskip-\thickarrayrulewidth
             \fi
      \ifnum0=`{\fi}}
\makeatother

\newlength{\thickarrayrulewidth}
\newcommand{\argmax}{\operatornamewithlimits{argmax}}
\newcolumntype{P}[1]{>{\centering\arraybackslash}m{#1}}
\newtheorem{theorem}{Theorem}

\begin{document}

\title{Optimal Scalability-Aware Allocation of Swarm Robots: \\From Linear to Retrograde Performance via Marginal Gains}

\author{
Simay Atasoy Bing\"ol$^{1,2\,\dagger}$,
Tobias T\"opfer$^{1}$,
Sven Kosub$^{1}$,
Heiko Hamann$^{1,2}$,
Andreagiovanni Reina$^{1,2,3\,\dagger}$%
\\[0.5em]
$^{1}$Department of Computer and Information Science, Universit\"at Konstanz, Germany\\
$^{2}$Centre for the Advanced Study of Collective Behaviour (CASCB), Universit\"at Konstanz, Germany\\
$^{3}$Department of Collective Behaviour, Max Planck Institute of Animal Behavior, Konstanz, Germany\\
$^{\dagger}$\texttt{simay.atasoy@uni-konstanz.de}; \texttt{andreagiovanni.reina@uni-konstanz.de}
}

\maketitle

\begin{abstract}
In collective systems, the available agents are a limited resource that must be allocated among tasks to maximize collective performance. Computing the optimal allocation of several agents to numerous tasks through a brute-force approach can be infeasible, especially when each task's performance scales differently with the increase of agents.
For example, difficult tasks may require more agents to achieve similar performances compared to simpler tasks, but performance may saturate nonlinearly as the number of allocated agents increases. We propose a computationally efficient algorithm, based on marginal performance gains, for optimally allocating agents to tasks with concave scalability functions---including linear, saturating, and retrograde scaling---to achieve maximum collective performance. We test the algorithm by allocating a simulated robot swarm among collective decision-making tasks, where embodied agents sample their environment and exchange information to reach a consensus on spatially distributed environmental features. We vary task difficulties by different geometrical arrangements of environmental features in space (patchiness). In this scenario, decision performance in each task scales either as a saturating curve (following the Condorcet's Jury Theorem in an interference-free setup) or as a retrograde curve (when physical interference among robots restricts their movement). Using simple robot simulations, we show that our algorithm can be useful in allocating robots among tasks. Our approach aims to advance the deployment of future real-world multi-robot systems.

\end{abstract}
         
\newcommand{\BibTeX}{\rm B\kern-.05em{\sc i\kern-.025em b}\kern-.08em\TeX}

\section{Introduction}
Groups can outperform individuals, if appropriately sized.
Task allocation plays a crucial role in both natural and artificial collective systems, enabling groups to work together efficiently. The cornerstone of social insects' success is their extraordinary ability to organize workers into groups specialized in performing one of the tasks needed for the colony's survival (e.g., foraging, nursing, or defense)~\cite{ROBINSON2009297}. Multi-robot systems can also exploit the parallelization of work by allocating teams of robots to different tasks (e.g., resource collection, transportation, and storage)~\cite{gerkey04,berman2009optimized}.

In such systems, ensuring efficient allocation of workers---animals or robots---to different tasks is crucial to optimize collective performance. In real-world applications, workers are provided in finite numbers and allocating them optimally across multiple tasks is non-trivial. Adding workers to a task could either improve performance, lead to diminishing returns, or even degrade overall performance if the resulting group is too large to operate efficiently. Given the finite set of workers, allocating more workers to one task requires reducing the number of workers assigned to other tasks. This tradeoff in resource distribution plays a crucial role in selective attention~\cite{paulk2014selective,tipper1988negative}, through which both individual and collective organisms allocate their resources (e.g., energy, time, or workers) to the most critical tasks~\cite{sasaki2013ants}.

In this work, we propose an efficient strategy for optimally allocating~$N$ agents among~$T$ tasks characterized by different scalability functions, defining how task performance changes for different numbers of agents allocated to it (see Fig.~\ref{fig:graph_abst} for an overview of our method). Our analysis shows that as the number of agents and tasks increases, a combinatorial explosion prohibits to search exhaustively for the optimal allocation. We propose an algorithm that has reduced computational complexity (polynomial time) compared to the brute-force approach, yet generates the optimal robot allocation.
We assume that tasks are spatially separated and task switching is costly. Hence, agents (e.g., robots) must be allocated offline, that is, before they are deployed and head to their operation areas to execute the tasks in parallel.

Similar allocation processes have been investigated in several disciplines, for example, in ecology, `ideal free distribution' describes how animals distribute themselves across different habitats~\cite{fretwell1969territorial,kacelnik1992ideal,quijano2007ideal}; in swarm robotics, `task allocation' algorithms enable robots to allocate themselves to tasks as a function of the task’s requirements~\cite{lerman2006analysis, berman2009optimized,kanakia2016modeling,fleming2019recruitment,prorok2017impact}; and in game theory, `hedonic games' model the splitting of players into coalitions~\cite{jang2018anonymous, bogomolnaia2002stability}. In the literature on multi-robot task allocation, it is generally assumed that tasks require a minimum number of agents~$n_\text{min}$ to achieve successful completion~\cite{prorok2017impact,aziz2021multi, aziz2022task}. When this requirement is not met ($n<n_\text{min}$), the task is considered unhandled. 
This leads to a binary modeling of task completion~$C$, where the robots' task achievement is classified as completed ($C(n)=1$ if $n\ge n_\text{min}$) or not executed ($C(n)=0$ if $n< n_\text{min}$). Even though this abstraction simplifies analysis and algorithm design, it may overlook intermediate cases where already a single agent could implement some progress ($0<C(1)\ll 1$). To address this previous limitation, we consider a more general task allocation scenario where task progress is determined by scalability functions~$C(n)$, that is, task performance gradually varies with the number~$n$ of assigned agents.

Prior studies on resource and budget allocation have addressed diminishing returns using submodular optimization techniques~\cite{bachrach2008distributed,soma2014optimal}, primarily to tackle scalability and strategic behavior in systems with limited resources. However, these approaches are restricted to settings with monotonic diminishing returns. In contrast, our work focuses on a broader class of concave scalability functions, including non-monotonic (retrograde) cases that fall outside the scope of standard submodular optimization frameworks. 

Frequency-dependent selection (FDS) in evolutionary biology provides a complementary perspective on agent allocation, where individual payoffs depend on how many others adopt the same strategy, leading in the case of negative FDS to diminishing returns and equilibrium behavior. Examples include the classic evolutionary game known as producer-scrounger model, in which 
individuals can either search for resources themselves (producers) or exploit resources discovered by others (scroungers), 
scrounging payoffs decrease as more individuals choose it~\cite{barnard1981producers}. Ideal free distribution theory predicts that animals distribute across resource patches to equalize intake rates under crowding~\cite{cantrell2007evolution}. These mechanisms are similar to our task allocation framework, where we assume concave scalability functions that induce diminishing marginal gains as more agents are assigned to a task. Both evolutionary and game-theoretic models reach equilibria in which no unilateral reassignment improves performance, similar to equilibria in potential games and evolutionary stable strategies~\cite{monderer1996potential}. These equilibria are typically reached through decentralized adaptation over time. In contrast, our approach explicitly computes the globally optimal allocation in polynomial time under known scalability functions, which makes it particularly well suited for engineered multi-agent systems such as robot swarms.

\begin{figure*}[h]
  \centering
  \includegraphics[width=1\linewidth]{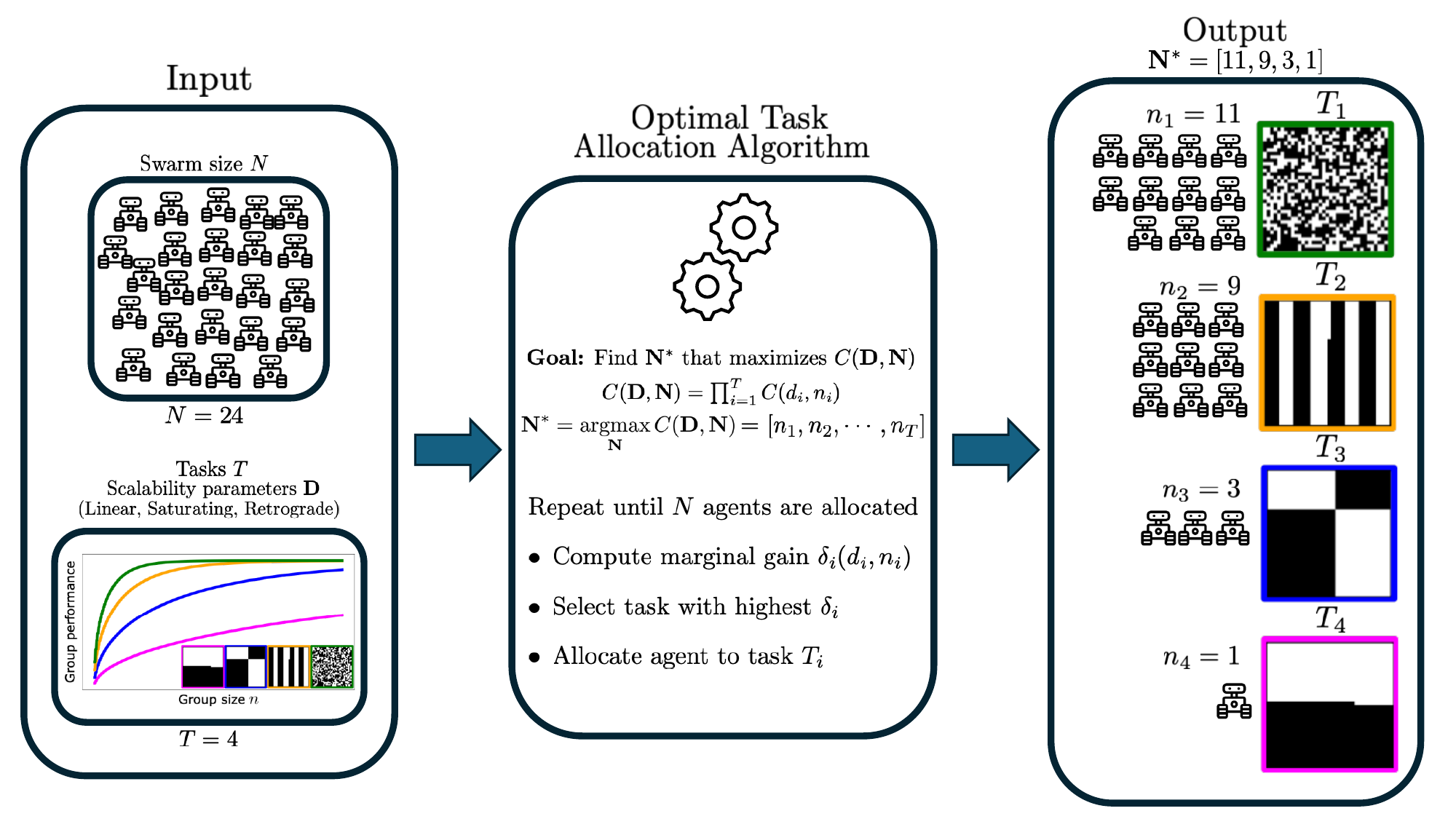}
  \caption{Overview of the proposed algorithm, showing the steps from input to final agent–task assignment. In this representative example, a total of $N = 24$ agents are allocated across $T = 4$ tasks, each characterized by a distinct scalability curve~$C(d_i,n_i)$. Based on marginal performance gains, the algorithm iteratively computes the optimal allocation as $N^* = [11, 9, 3, 1]$. In the illustrated collective decision-making scenario (collective sensing of environmental features shown as black and white floor tiles), the robot allocation is biased in favor of the easiest tasks, that is, tasks where robots can measure the correct state of the world with increased individual accuracy. In contrast, when swarms are large (resources to be allocated are abundant), the optimal solution is biased in favor of the more difficult tasks (Sec.\ref{sec:res-many-tasks}).}
  \label{fig:graph_abst}
\end{figure*}

Our algorithm generalizes to task allocation for any concave scalability functions (see Fig.~\ref{fig:perf_all}), which are frequently found both in computational and swarm robotics systems, as both face similar scalability constraints~\cite{hamann2021scalability}.
In parallel computing, performance is limited by bottlenecks, such as memory bandwidth or CPU constraints, resource contention when multiple processes compete for shared resources, and process synchronization overhead, which slows execution due to coordination delays. Similarly, in swarm robotics, scalability is constrained by physical interference among robots and with the environment, competition for shared resources such as communication bandwidth or physical space, and coordination overhead to create an efficient collective behavior. Building on this similarity, we consider three types of scalability functions that are most found in the literature, which exhibit linear, saturating, and retrograde trends.

We validate our algorithm with multi-agent and robotics simulations.
In our experiments, we model tasks as collective decision-making problems, which are a common application and benchmark in swarm robotics~\cite{hamann2018swarm}. 
To make consensus decisions on the true environmental state, agents need to collect information, process it, and collectively select one option.
As agents can only sample locally, they acquire inaccurate, incomplete, or non-representative environmental information. However, if we deploy several individuals who take and combine independent samples from the same area, we can reduce variance and improve decision accuracy. 

Combining independent samples to increase the precision of an estimate is known as variance reduction~\cite{kroese2013handbook}. Variance reduction is common in natural and artificial collective systems that exploit the plurality of the individuals to maximize the quality of sampled information on which the decision will be based~\cite{sasaki2013ants, kao2014decision, khaluf2019collective}. 

Thus, our task allocation scenario can be seen as a form of \textit{collective selective attention}, where the swarm improves the decentralized acquisition of environmental information by allocating individuals to a set of sampling tasks.

Collective performance can be improved by allocating more agents to more difficult tasks, that is, areas where agents make more frequent observation errors due to a higher variance in their individual observations. The intuition is that by investing more resources (i.e., more agents) into acquiring environmental information, observational variance is reduced. A~similar variance-reduction technique has been observed in cognitive neuroscience studies of human perceptual decisions with two stimuli (equivalent to our spatial features). The participant's resource was time and the most efficient solution was to allocate more time to categorize the noisier (i.e., higher variance) stimulus~\cite{cassey2013adaptive}. Similarly, another study found that ants adaptively vary their nest-site selection strategy to allocate more attention (resources) on higher variance attributes, which are more difficult to be categorized~\cite{sasaki2013ants}. In statistics, the idea of using sampling groups with sizes proportional to the stimulus variance is formalized as the importance sampling technique~\cite{glynn1989importance,neal2001annealed}.  

In our collective decision-making scenario, agents combine social information through majority votes, hence we can measure the scaling of performance through the Condorcet's Jury Theorem (CJT). 
This theorem states that, in a binary decision problem, assuming each agent votes independently and has a greater tendency to vote for the correct rather than the incorrect decision, the probability of the majority decision being correct (i.e., the group accuracy) increases with the group size~\cite{condorcet1785essay}. 
In addition, we test our algorithm in a series of simple robot simulations, indicating the relevance of our solution for real-world problems as seen in large-scale robot systems.

Our results show that task scalability functions are not the only relevant parameters in defining the optimal allocation of $N$~agents to $T$~tasks. Instead, the proportion of agents allocated to each task that maximizes collective accuracy is size-dependent. This means that the optimal allocation of resources is not a constant proportion of the swarm size, rather the relative team sizes are different when the swarm has fewer or more agents. This size-dependent allocation is only present in nonlinear scalability functions, where collective performance saturates to a maximum value when a large number of agents are assigned to a task.

Our main contributions are as follows: (i) we formally characterize the multi-task allocation problem and show the naive exhaustive approach has exponential algorithmic complexity; (ii) we propose a polynomial-time algorithm adapting marginal-gain optimization to compute the optimal allocation for any concave scalability function describing how performance depends on group size; (iii) we exploit this algorithm to study how optimal solutions change with swarm size $N$ and number of tasks $T$, revealing counter-intuitive cases in which assigning more agents to the most difficult task is suboptimal for small swarms; and (iv) we validate the relevance of our approach by applying the algorithm to swarm robotics scenarios and demonstrating its effectiveness through extensive multi-agent simulations.

\section{Multi-task Allocation Problem} \label{sec:prob_des}

We focus on the problem of allocating $N$ identical agents to a set of $T<N$ tasks to achieve maximum collective performance. Each agent can only be allocated to a single task. Maximum performance is achieved by maximizing the performance in each task. Failure in one task means the entire mission fails. 

Our algorithm relies on the following assumptions:
\begin{itemize}
    \item The task allocation is computed by a central unit with full knowledge of individual task performance curves and agent availability.
    \item Agents are homogeneous; they have identical capabilities.
    \item Tasks are independent of each other; there are no interactions or interference effects between tasks.
    \item The scalability function of every task, defining how the task performance changes with respect to the number of allocated agents, is concave.
    \item Agents are allocated to tasks prior to execution; there is no task switching or reallocation during runtime.
   \item Each task is of equal importance in the allocation framework.
\end{itemize}

\paragraph{Collective performance} We formalize the multi-task allocation problem as follows: the vector $\mathbf{N}=(n_1,\dots,n_T) \in \mathbb{N}^T_{>0}$, with $\sum^T_{i=1} n_i = N$ represents how the $N$ agents are distributed among $T$ tasks and $\mathbf{D} = (d_1,\dots,d_T)$ represents the parameters for the scalability curves of all tasks.
We measure the collective performance as the product of performances of all individual tasks:
\begin{equation}
\label{eqn:mult}
    C(\mathbf{D},\mathbf{N}) = \prod_{i=1}^T C(d_i , n_i) \;.
\end{equation}
Here, $C(\cdot)$ represents the scalability function, and $C(d_i,n_i)$ gives the individual task performance with scalability parameters $d_i$ and $n_i$ agents allocated to it.

Computing the collective performance as a multiplication of $C(d_i,n_i)$ assumes that all tasks are of equal importance. While this work does not explicitly address the scenario when some tasks are more critical than others, our approach can also be applied seamlessly to a collective performance function that includes different task importances through a weighted sum.

\paragraph{Goal} 
Our goal is to find the vector $\mathbf{N}$ that maximizes Eq.~\eqref{eqn:mult} for a given~$\mathbf{D}$, which, in mathematical terms, is formalized as follows: 
\begin{equation}
\label{eqn:argmax}
    \mathbf{N}^* = \argmax_\mathbf{N} C(\mathbf{D},\mathbf{N}).
\end{equation}

To model swarm performance for a variety of tasks, we follow how scalability is measured and modeled in distributed computing systems. Computational performance in parallel systems depends on how shared resources are managed and is influenced by bottlenecks and contention. Performance in swarm systems shows similar scalability constraints due to physical interference and communication overhead~\cite{hamann2021scalability}. Building on this similarity, we model the different scalability functions (i.e., different tasks) to show linear, saturating, and retrograde trends (see Fig.\,\ref{fig:perf_all}). We summarize the structure of the proposed method in Fig.\,\ref{fig:graph_abst}.

\paragraph{Linear scalability}
Linear scalability refers to the ability of the system to increase its throughput (i.e., performance) proportionally to the number of resources (i.e., agents) added (blue curve in Fig.\,\ref{fig:perf_all}). We use Gustafson's Law~\cite{gustafson1988} (GL) to model tasks which are highly parallelizable, displaying a (near-)linear relationship between the number of agents and the task performance. This relationship is caused by the assumption that most work of the task can be performed independently, in parallel, without any interference by each agent. 

We formalize GL as 
\begin{equation}
\label{eqn:gl}
    C_\text{GL}(\lambda, n) = n-\lambda(n-1) \;,
\end{equation}
where $n$ is the number of agents allocated to the task and $\lambda\in(0,1)$ is the fraction of work that cannot be parallelized. 

In swarm robotics, we can model tasks, such as foraging and surveillance, as GL because, typically, their performance (e.g., number of collected items or area coverage) scales linearly with the number of robots, assuming the environment and object availability/distribution grows proportionally~\cite{balch2000reward}.

\paragraph{Saturating scalability}
Systems that exhibit saturating scalability increase their performance as resources are added, but eventually plateau at a maximum performance level, regardless of how many additional resources are introduced  (orange curve in Fig.\,\ref{fig:perf_all}). In swarm robotics, collective decision-making usually shows saturating scalability \cite{zakir2024miscommunication}, which we model using the Condorcet's Jury Theorem (CJT). The CJT gives the probability of the majority of a group of~$n$ decision-makers having the correct opinion in binary decisions. 
The theorem assumes that each agent makes its decision independently (i.e., without being influenced by other group members and acquiring uncorrelated external information). For example, robots deciding on the current state of the environment make independent observations from different locations. Every agent makes the correct personal decision with a probability~$p$. 

The CJT indicates that if the personal decisions are aggregated with a majority rule, that is, the collective decision is the one chosen by the majority of the agents, the average collective accuracy increases as
\begin{equation}
\label{eq:cjt}
\begin{array}{ll}
C_\text{CJT}(p,n) = &\sum_{k=\lfloor\frac{n}{2} + 1\rfloor}^{n} \binom{n}{k} {p}^k (1-p)^{n-k} \\ &+ \frac{1}{2} \binom{n}{\frac{n}{2}} p^{\frac{n}{2}} (1-p)^{\frac{n}{2}}(1-n\text{~mod~}2) \;,
\end{array}
\end{equation}
where the term $\binom{n}{k}$ represents the binomial coefficient and $\lfloor.\rfloor$ the floor operator (rounds down to the nearest integer). Eq.~\eqref{eq:cjt} is a sum with the last term being non-zero only when~$n$ is an even number. This term covers the probability that ties occur (i.e., an equal number of votes for both options), which are resolved randomly (50-50\% decision).

\paragraph{Retrograde scalability}
In systems exhibiting retrograde scalability, performance initially increases with the addition of resources, it eventually achieves a peak, and beyond that point, it degrades as more resources are added to the system  (green curve in Fig.\,\ref{fig:perf_all}). We use the Universal Scalability Law (USL) to model tasks with retrograde scalability~\cite{gunther93}, using two parameters: the degree of contention $\alpha$ (i.e., competition for shared resources) and the lack of coherence $\beta$ (i.e., overhead from maintaining consistency and coordination). 

The USL for a swarm of size $n$ is 
\begin{equation}
\label{eqn:usl}
    C_\text{USL}(\alpha, \beta, n) = \frac{n}{1+ \alpha (n-1)+ \beta n (n-1)} \;.
\end{equation}
The USL can also represent linear, saturating, and superlinear scalability (e.g., for $\alpha<0$)~\cite{gunther15}. In this work, we only focus on the retrograde case by setting $\beta\ge0$.

In swarm robotics, physical interference between robots affects performance as the swarm size increases, resulting in diminishing returns or performance degradation in larger swarms~\cite{hamann2021scalability,kuckling24}.

\paragraph{Task scaling} 
The interpretation of the scalability parameter~$d_i$ varies depending on the scalability function. For task $i$ with a linear scalability curve, $d_i$ is equal to $\lambda_i$, which defines the slope of the line as $1-\lambda_i$. In the CJT-based models, $d_i$ corresponds to the agent's probability $p_i$ of making a correct decision in task $i$. If task $i$ exhibits retrograde scalability (modeled using USL), $d_i$ represents the pair of parameters $\alpha_i$ and $\beta_i$ expressed as $d_i = (\alpha_i, \beta_i)$.

\begin{figure}[t]
  \centering\includegraphics[width=0.7\linewidth]{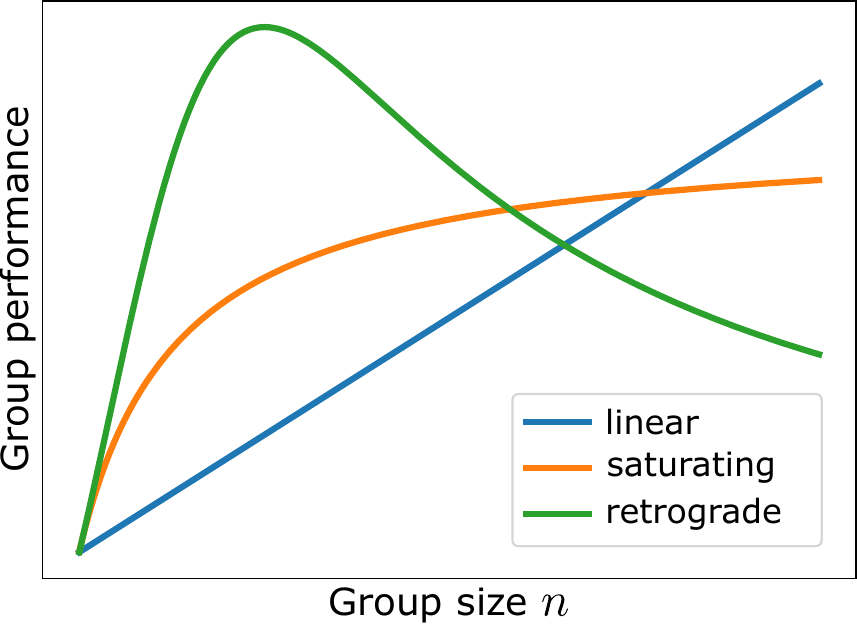}
  \caption{Schematic plot of three types of scalability functions---with linear, saturating, and retrograde trends---describing the group performance as a function of group size. Each task, depending on the particular application and scenario, can scale differently.}
  \label{fig:perf_all}
\end{figure}

\section{Modeling Majority Decision-Making}
\label{sec:case-study}

In this section, we show that collective decision-making tasks in swarm robotics can exhibit either saturating or retrograde scalability, depending on the scenario constraints and assumptions. When there is no physical interference between robots (or virtual agents), collective decision-making via majority rule is well captured by the saturating CJT function (Eq.\,\eqref{eq:cjt}). In contrast, when robots experience physical interference, swarm performance follows retrograde scalability and is well described by the USL (Eq.\,\eqref{eqn:usl}). 

We model these collective decision-making tasks using one of the standard benchmarking scenarios in swarm robotics~\cite{valentini2016collective,ebert2020bayes, zakir2022robot}. Robots must reach a consensus on the prevalent color of the floor. Our robots operate in environments with a floor composed of black and white tiles. Despite having access only to local and incomplete information, they combine each other's estimates to reach a consensus on the predominant color. 

\paragraph{Simulated environment} The simulated environment of each task has a size of $36\times36$ space-units (su). We also show that using environments with different sizes does not affect the results (see larger arena simulation results in Supplementary Fig.~S1). The binary classification task is characterized by both the proportion of black and white tiles (fill ratio $f$) and the spatial distribution of colors (e.g., large clusters of tiles of one color). In particular, we consider five fill ratios $f \in \{0.51, 0.52, 0.53, 0.54, 0.55\}$ and four spatial distributions---checkerboard, striped, four rectangles, and halved environments---illustrated in Fig.~\ref{fig:env}. Our choice of environments is motivated by the benchmarking framework proposed by Bartashevich et al.\cite{bartashevich2019benchmarking} which introduced different spatial patterns to capture task difficulty in collective perception (see Supplementary Sec.~2 and Supplementary Figs.~S2 - S3 for additional results and discussion on the specific choice of fill ratios). In the checkerboard environment, tiles of square size $1\time1$ su are placed uniformly at random. 
In the striped environment, bars of a uniform color with size $5\times36$~su are alternated (note that some bars are 1~su smaller to achieve the exact fill ratio~$f$). 
In the four rectangles environment, the floor is divided into four large quadrants with a uniform alternated color in each of them. In the halved environment, each half has a single color (note that the partitioning line is not precisely at half of the environment but is set to match the desired fill ratio~$f$).

\begin{figure}[t]
  \centering\includegraphics[width=0.85\linewidth]{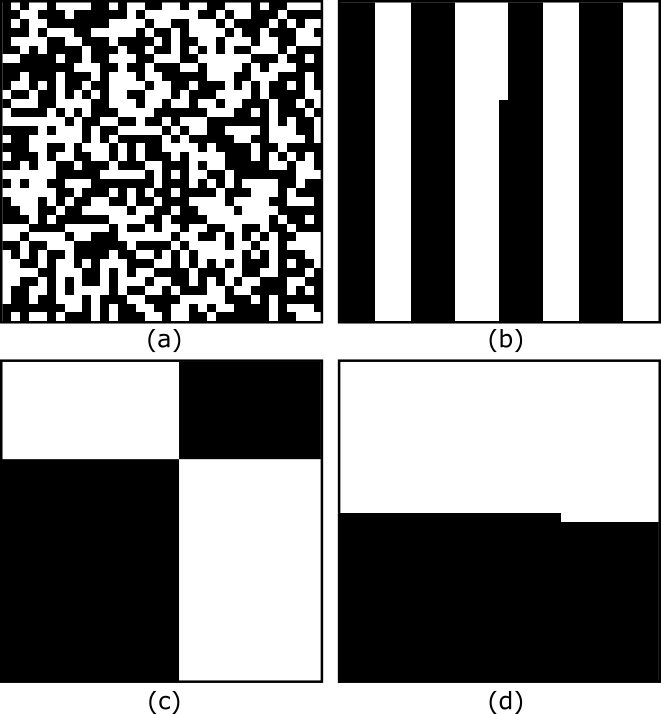}
  \caption{Simulated binary classification tasks with fill ratio $f=0.52$. The environment floor (sized \unit[$36\times36$]{su}) is composed of black and white tiles (sized \unit[$1\times1$]{su}). The robots are tasked with reaching a consensus on the most frequent color. We consider four spatial distributions: (a)~checkerboard, (b)~striped, (c)~four rectangles and (d)~halved environments.}
  \label{fig:env}
\end{figure}

\paragraph{Simulated robots}
Robots are modeled as particles that move in a 2D environment. In a first set of experiments, we assume no collisions among robots---that is, they are treated as dimensionless points. This eliminates physical interference and allows for ideal scalability, where adding more robots never impedes their movement. In Sec.~\ref{sec:interference}, we also consider a more realistic scenario in which robots' physical embodiment affects one another's movement. In this second case, we observe qualitatively different scaling of task performance, exhibiting a retrograding trend when too many robots are deployed in the same confined environment.

Robots can move, exchange opinion votes with one another, and are equipped with proximity and ground sensors. 
The robots move at a speed of \unit[0.1]{su} per timestep and rotate on the spot at an angular speed of $3^\circ$ per timestep. Proximity sensors are three binary sensors to detect the presence of obstacles in front of the robot at a distance smaller than \unit[4.5]{su}. One sensor faces the robot's heading direction ($0^\circ$) and the other two at $\pm45^\circ$. Proximity sensors are used to detect walls surrounding the environment and (in embodied simulations) other robots. Once an obstacle is detected, the robot initiates an avoidance maneuver, which consists of rotating in place until there is no obstacle in the detection range. Ground sensors allow robots to detect the color (black or white) of the tile beneath them. To avoid redundant sampling from the same floor tile, they track their covered distance and take two subsequent floor samples at a distance of \unit[1]{su} apart. Given the robots' limited motion speed and sensing range, they often make inaccurate estimates of the state world. However, they can combine their opinions with others to improve group accuracy.

\subsection{Three Robot Controllers}
\label{sec:control}

\begin{figure}[t]
  \centering
  \includegraphics[width=0.95\linewidth]{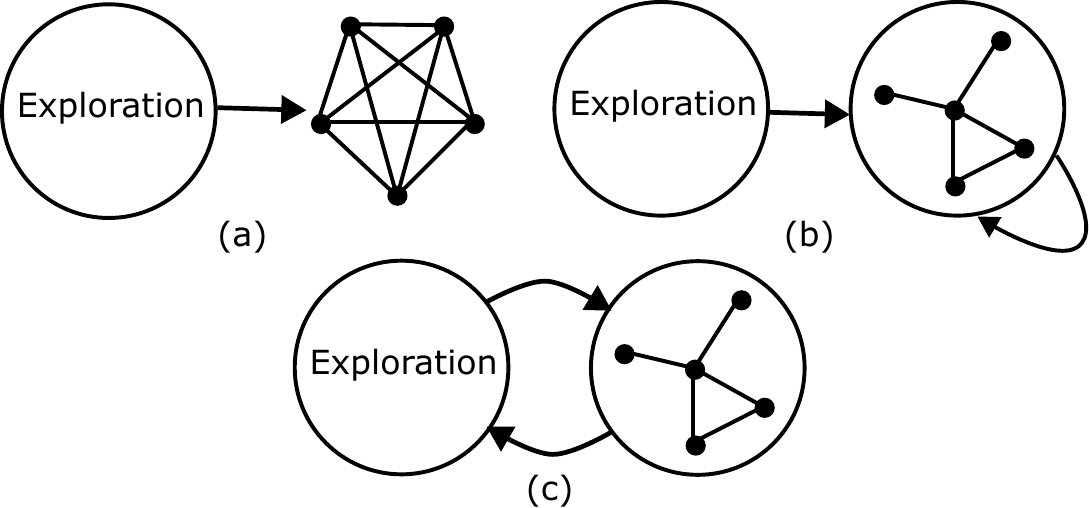}
  \caption{(a) Centralized controller: All robots start in the exploration state, during which they gather environmental data and form their initial opinion. After exploring, they communicate through all-to-all interactions to reach a majority-based consensus. (b)~Decentralized controller: Robots form their initial opinion similarly to the centralized controller. However, communication is restricted to local interactions and robots disseminate their opinions to neighbors until the entire swarm reaches a unanimous decision. (c)~Iterative controller: This controller follows an iterative process of exploration and dissemination where they repeatedly combine personal and social information. Robots simultaneously disseminate their opinions, listen to neighbors, and continue exploring. The iterative process continues until the swarm reaches a unanimous decision.}
  \label{fig:controllers}
\end{figure}

We consider three alternative controller algorithms to enable the robots to perform the sampling tasks collectively: centralized, decentralized, and iterative (see Fig.\,\ref{fig:controllers}). The centralized controller relies on all-to-all communication for a single-step majority decision, the decentralized controller spreads opinions locally until consensus emerges, and the iterative controller continuously refines opinions by integrating both individual and social information through repeated exploration and dissemination.

The \textbf{centralized} controller (Fig.\,\ref{fig:controllers}a) is the closest application of the CJT. All robots first explore the environment through a random walk to form their initial opinion on the state of the world (prevalence of white or black tiles). The exploration time is different for each robot and is drawn from a uniform distribution with a mean of 1200 timesteps to model asynchronous robot operations. Then, each robot votes for their opinion to make a majority decision. Robots are in an all-to-all communication network. Hence, all robots have global information, aggregate the opinions of the entire swarm, and reach the same decision. 

The \textbf{decentralized} controller (Fig.\,\ref{fig:controllers}b) uses the same initial opinion formation process as the centralized controller; however, the communication range of each robot is restricted to $7.5$~su. Robots' local interactions combined with their random walk in the environment lead to opinion spreading on a sparse time-varying communication network. Therefore, each robot only receives a part of others' opinions and consensus cannot always be reached in a single step. Hence, robots repeatedly make the majority decision only using the opinions of neighbors in their communication range. This process continues until the swarm achieves a unanimous consensus.

The \textbf{iterative} controller (Fig.\,\ref{fig:controllers}c) extends the decentralized controller by iteratively combining personal and social information. At the start, there is no social information available yet, hence each robot---in the same fashion as the other controllers---forms its opinion through environmental exploration and shares it with its neighbors. Robots repeatedly explore and sample the environment by making (local) majority decisions where one of the processed votes is the result of their last individual observation of the environment. This majority decision forms the robot's new opinion which is then broadcast to its neighbors. This iterative process continues until the swarm reaches a unanimous consensus.

\subsection{Majority Decision Results}
\label{sec:majority_des}
We conduct a series of multi-agent simulations to compare the theoretical prediction based on the CJT with the results obtained from our simulations, where agents run one of the three robot controllers of Sec.~\ref{sec:control}. 
These are simple robot simulations implemented using a Pygame-based simulator.\footnote{Swarmy: Robot swarm simulator \url{https://github.com/tilly111/swarmy}}
For the simulations studied in this section, agents have no volume and do not collide with each other (in contrast to Sec.~\ref{sec:interference} below).

As mentioned above, our simulated case study is the binary floor-color classification task which is a standard benchmark in swarm robotics~\cite{valentini2016collective,ebert2018multi, zakir2022robot}. We conduct 250~independent experiments for each environment and swarm size $N \in \{1,\dots,29\}$. We consider five fill ratios $f \in \{0.51, 0.52, 0.53, 0.54, 0.55\}$ for the checkerboard environment and four spatial distributions with the same fill ratio $f = 0.52$ (see Fig.~\ref{fig:env}).
To compare the multi-agent simulations with the CJT predictions, we measure the individual accuracy probability~$p$ of any agent to make the correct environmental estimate during the allocated exploration time interval for each environment type. We measure~$p$ as the average of all individual agent estimates of our experiments (i.e., approximately $10^5$ data points). Table~\ref{tab:individualAcc} illustrates the relationship between the estimated individual accuracy~$p$ and the different types of environments. 

As already documented in previous studies~\cite{valentini2016collective,Shan2020, bartashevich2021multi}, the task becomes harder and the individual accuracy~$p$ decreases either by decreasing the fill ratio (pushing it closer to $f=0.5$) or by increasing the spatial correlation of colors. 
In addition, the value $p$ also depends on the behavioral and sensing capabilities of the robots, such as the type of random walk they perform or how frequently they sample the environment, as well as the coverage they achieve during exploration. Increased coverage of the arena results in more representative samples and thus increases the probability of making a correct decision.
By decreasing the fill ratio, we reduce the difference in frequency between white and black tiles, therefore the probability~$p$ of making correct individual observations is reduced. It may seem less intuitive why we observe also a reduced probability~$p$ when we keep the fill ratio unchanged and only increase the spatial correlations of colors (clusters of uni-color tiles and hence more patchy environment). Agents taking local samples in highly correlated environments are less likely to sample representative environmental data. This spatial correlation analysis shows that embodied agents (e.g., robot swarms) are exposed to several factors that can cause inaccuracies in their individual observations, characterizing the task scalability curve.

\begin{table} [tb]
\renewcommand{\arraystretch}{1.2}
   \medskip
   \caption{Individual agent accuracy for the considered environments.}
   \label{tab:individualAcc}
   \begin{center}
   \begin{tabular}
     {P{1.3cm}| P{1.2cm} P{1.2cm} P{1.2cm} P{1.2cm}}
       & \multicolumn{4}{c}{estimated individual probability $p$} \\
     fill ratio $f$ & checkerb. & striped & four rect. & halved \\\hline\thickhline
     \\[-3ex]
	0.51 & \begin{tabular}{@{} P{1.2cm} P{1.2cm} P{1.2cm} P{1.2cm}} 0.5361 & / & / & /\end{tabular}\\
	\hline
	0.52 & \begin{tabular}{@{} P{1.2cm} P{1.2cm} P{1.2cm} P{1.2cm}} 0.6017 & 0.5698 & 0.5402 & 0.5177 \end{tabular}\\
	\hline
	0.53 & \begin{tabular}{@{} P{1.2cm} P{1.2cm} P{1.2cm} P{1.2cm}} 0.6603 & / & / & /\end{tabular}\\
	\hline
	0.54 & \begin{tabular}{@{} P{1.2cm} P{1.2cm} P{1.2cm} P{1.2cm}} 0.7454 & / & / & /\end{tabular}\\
 \hline
    0.55 & \begin{tabular}{@{} P{1.2cm} P{1.2cm} P{1.2cm} P{1.2cm}} 0.8069 & / & / & /\end{tabular}\\
   \end{tabular}
   \end{center}
\end{table}

Given that individual accuracies~$p_i$ for any considered task~$i$ are larger than the chance level ($\forall i,  p_i>0.5$), the CJT states that as the number $n_i$ of robots assigned to task~$i$ increases, the collective accuracy of majority decisions (Eq.\,\eqref{eq:cjt}) monotonically increases and asymptotically converges to one~\cite{king2007use}.
The results of Fig.~\ref{fig:single-task-fill}(a) show that for all three considered controllers, there is a good agreement between the CJT predictions (solid lines) and the multi-agent simulations (dashed/dotted lines) for the five considered fill ratios~$0.51\le f\le 0.55$. Hence, we can use the CJT of Eq.~\eqref{eq:cjt} to predict how the group performance scales for larger group sizes $n\in[1,500]$, as shown in Fig.~\ref{fig:single-task-fill}(b). Fig.~\ref{fig:single-task-dist} shows a similar analysis for the four different spatial distributions, showing again a good match between the theorem and the simulations. As there is an inverse relationship between the spatial correlation of the colors in the environment and the individual agent accuracy~$p$ (see Table~\ref{tab:individualAcc}), this trend is also shown at the group level with lower collective performance in environments with high correlated spatial features (Fig.\,\ref{fig:single-task-dist}). 
Nevertheless, in any tested environment, group performance increased with group size~$n$.

\begin{figure}[t]
  \centering
  \includegraphics[width=1\linewidth]{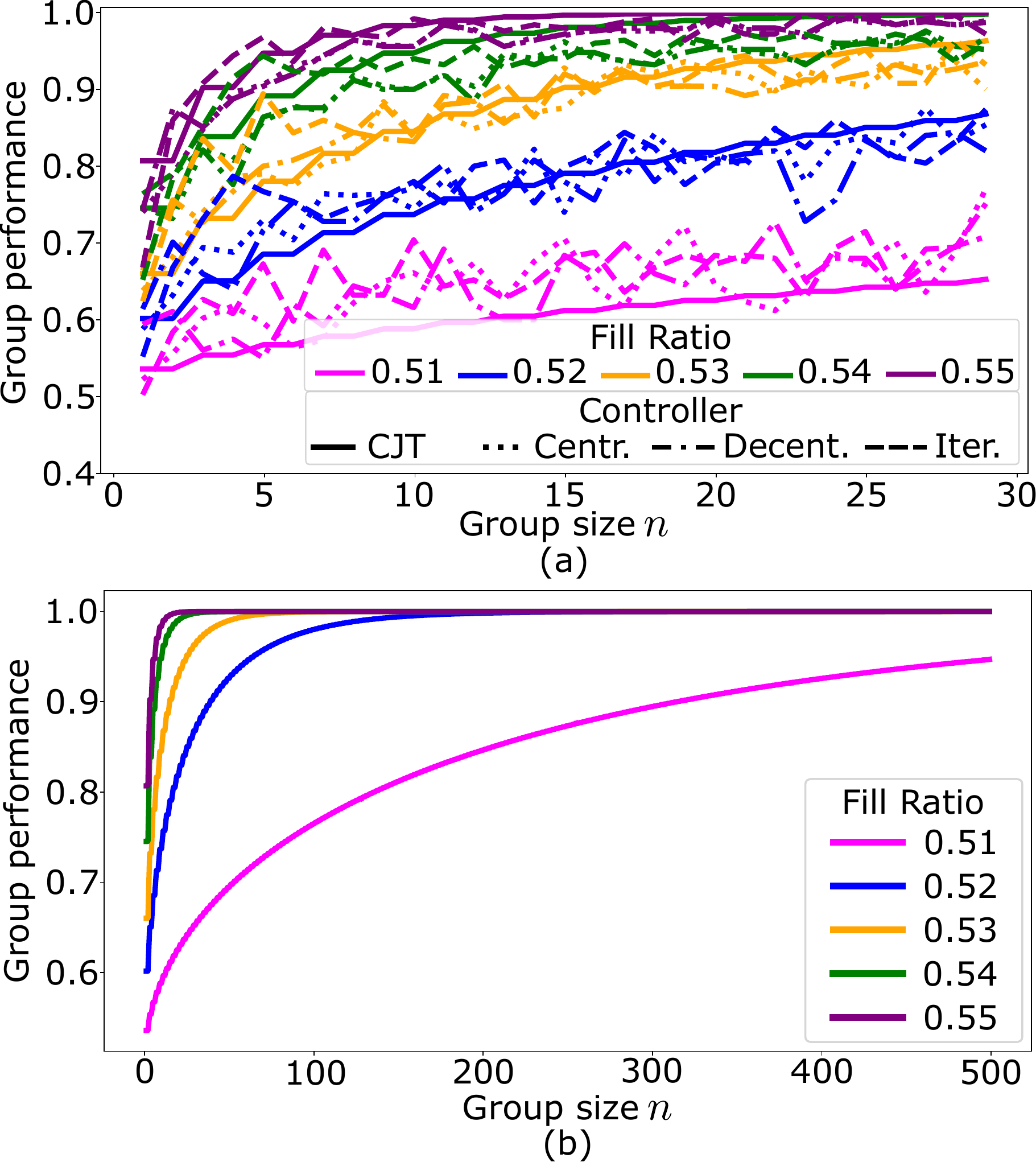}
  \caption{Scaling of the group accuracy (y-axis) in making majority decisions for different group sizes (x-axis) in the standard checkerboard environment (see Fig.~\ref{fig:env}(a)). The different colors show results for different fill ratios.  Solid line shows the CJT predictions and, in~(a), the dotted, dash-dotted, and dashed lines show the multi-agent results for the centralized, decentralized, and iterative controllers, respectively. There is a good agreement between the CJT predictions and the multi-agent results. In~(b), the CJT lines are extended to larger swarm sizes, using the same values of~$p$ as in~(a).}
  \label{fig:single-task-fill}
\end{figure}

\begin{figure}[t]
  \centering
  \includegraphics[width=1\linewidth]{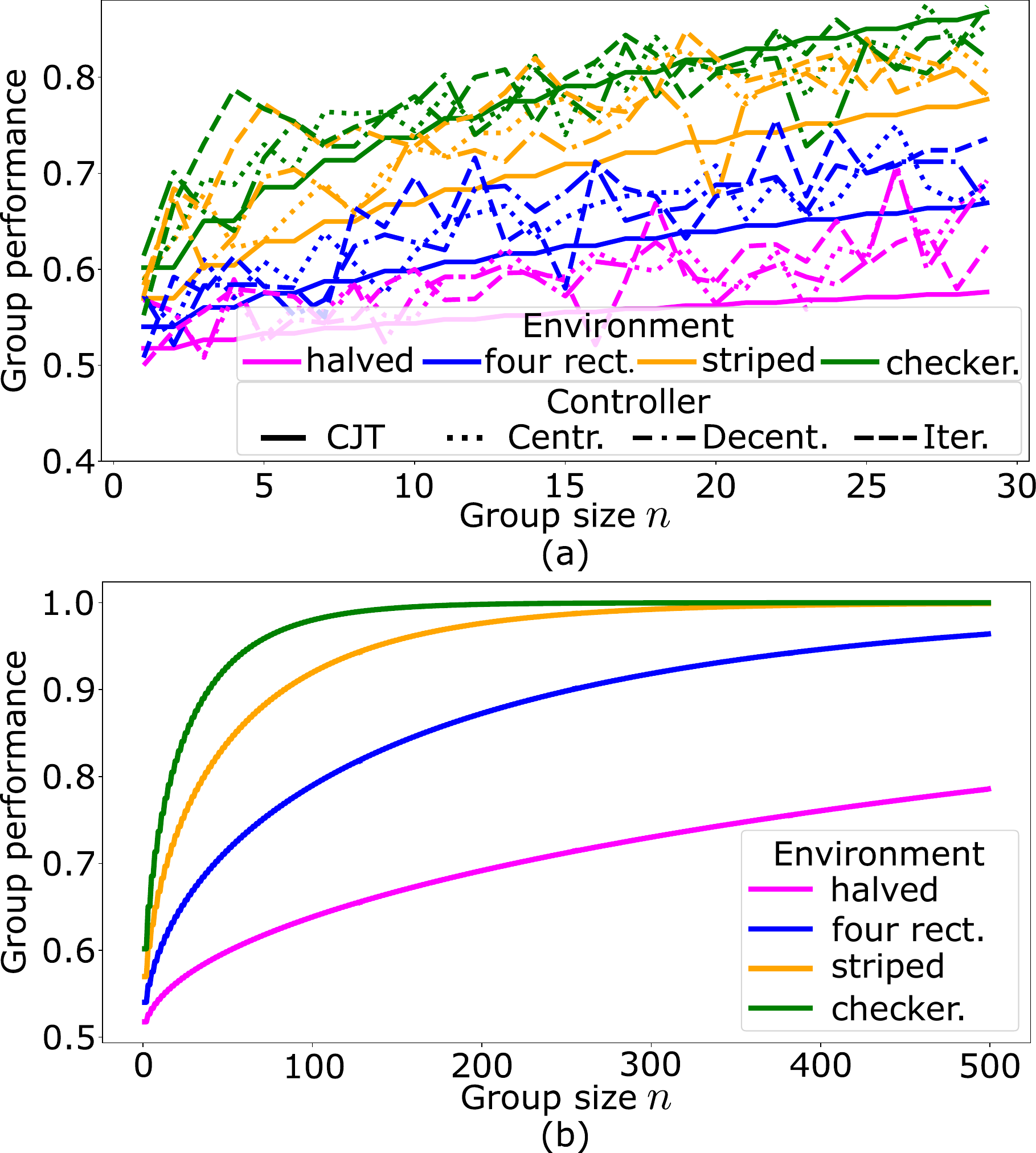}
  \caption{Results for the three spatially correlated environments: striped, rectangles, and halved (see Fig.~\ref{fig:env}(b, c, d)). Scaling of the group accuracy (y-axis) in making majority decisions for different group sizes (x-axis). The different colors show results for different spatial distributions and the line types match the ones used in Fig.~\ref{fig:single-task-fill}. As also shown in Table~\ref{tab:individualAcc}, when the color correlation increases, the task gets harder and the accuracy decreases.}
  \label{fig:single-task-dist}
\end{figure}

\subsection{Simulations with Physical Interference}
\label{sec:interference}
Above we assumed agents to be without volume that do not physically interfere with each other, meaning there were no collisions between agents restricting their movement (aside from interactions with walls). 
However, in real-world applications (e.g., applications of swarm robotics), such an idealized assumption cannot hold. 
Robots often operate in confined shared spaces where physical interference inevitably leads to collisions or even congestion that can significantly affect the robot system's efficiency. 
Therefore, it is not guaranteed that collective accuracy monotonically increases with swarm size. 
Beyond a critical swarm size, increasing the number of robots can actually reduce collective performance, as physical interference leads to traffic jams and bottlenecks~\cite{hamann2021scalability}. Such a critical swarm size is normally expressed in terms of robot density, that is, the number of robots per space unit.

Given the importance of physical interference in practical applications, we include robot-to-robot physical collisions and preemptive avoidance maneuvers in a modified version of the multi-agent simulator described in Sec.~\ref{sec:case-study}. Robots detect each other at a distance of $4.5$~su and initiate avoidance maneuvers. As a side effect of these maneuvers---consisting of in-place rotations---robots temporarily stop collecting new environmental samples, as sampling only occurs when the robot has moved at least 1~su from its previous sample location.

We ran a series of experiments in the checkerboard environment to show the effect of physical embodiment. 
As an emergent effect of robot-to-robot interference, we observe retrograde scalability. 
Fig.~\ref{fig:interference}(a) shows the multi-robot simulation results (dotted/dashed lines) with the fitted USL functions computed via nonlinear least squares fitting. Table~\ref{tab:USLparams} shows the fitted parameters for the various fill ratios $f \in \{0.51, 0.52, 0.53, 0.54, 0.55\}$. To normalize the curves to the interval~$[0,1]$, we introduce a proportionality constant~$k$:
$kC_\text{USL}(\alpha, \beta, n)$. 
We also provide the root mean square error (RMSE) as a metric to evaluate the goodness of fit. For small groups $n<10$, where robot densities and physical interference are low, group accuracy increases with $n$. 
However, for groups of approximately $n>10\sim15$ robots, group accuracy deteriorates as~$n$ increases due to detrimental physical interference, leading to retrograde scalability. This decline in group performance is also captured by the fitted USL curves: as a positive parameter $\beta>0$ (see Table~\ref{tab:USLparams}) indicates a retrograde scalability function. 
Note that the prediction of performance by the fitted USL shown in Fig.~\ref{fig:interference}(b) tends to be pessimistic, because it is unlikely that the performance~$C$ would go below the fill ratio~$f$ ($C<f$). 
As observed in Fig.~\ref{fig:interference}, although group accuracy decreases, the decline is slow and linear, indicating that robots tend to resolve interference and resume productive operation quickly~\cite{hamann2020guerrilla}. 

\begin{table} [tb]
\renewcommand{\arraystretch}{1.2}
   \medskip
   \caption{Estimated Parameters for USL-Based Scalability Curves for Interference Experiments}
   \label{tab:USLparams}
   \begin{center}
   \begin{tabular}
     {P{1.3cm}| P{1.2cm} P{1.2cm} P{1.2cm} P{1.2cm}}
      fill ratio $f$ & $\alpha$ & $\beta$ & $k$ & RMSE \\\hline\thickhline
     \\[-3ex]
	0.51 & \begin{tabular}{@{} P{1.2cm} P{1.2cm} P{1.2cm} P{1.2cm}} 0.7971 & 0.0012 & 0.5194 & 0.0305\end{tabular}\\
	\hline
	0.52 & \begin{tabular}{@{} P{1.2cm} P{1.2cm} P{1.2cm} P{1.2cm}} 0.6376 & 0.0021 & 0.5270 & 0.0325 \end{tabular}\\
	\hline
	0.53 & \begin{tabular}{@{} P{1.2cm} P{1.2cm} P{1.2cm} P{1.2cm}} 0.6750 & 0.0016 & 0.6093 & 0.0241 \end{tabular}\\
	\hline
	0.54 & \begin{tabular}{@{} P{1.2cm} P{1.2cm} P{1.2cm} P{1.2cm}} 0.7089 & 0.0010 & 0.6814 & 0.0231\end{tabular}\\
 \hline
    0.55 & \begin{tabular}{@{} P{1.2cm} P{1.2cm} P{1.2cm} P{1.2cm}} 0.7526 & 0.0003 & 0.7201 & 0.0204\end{tabular}\\
   \end{tabular}
   \end{center}
\end{table}

\begin{figure}[t]
  \centering
  \includegraphics[width=1\linewidth]{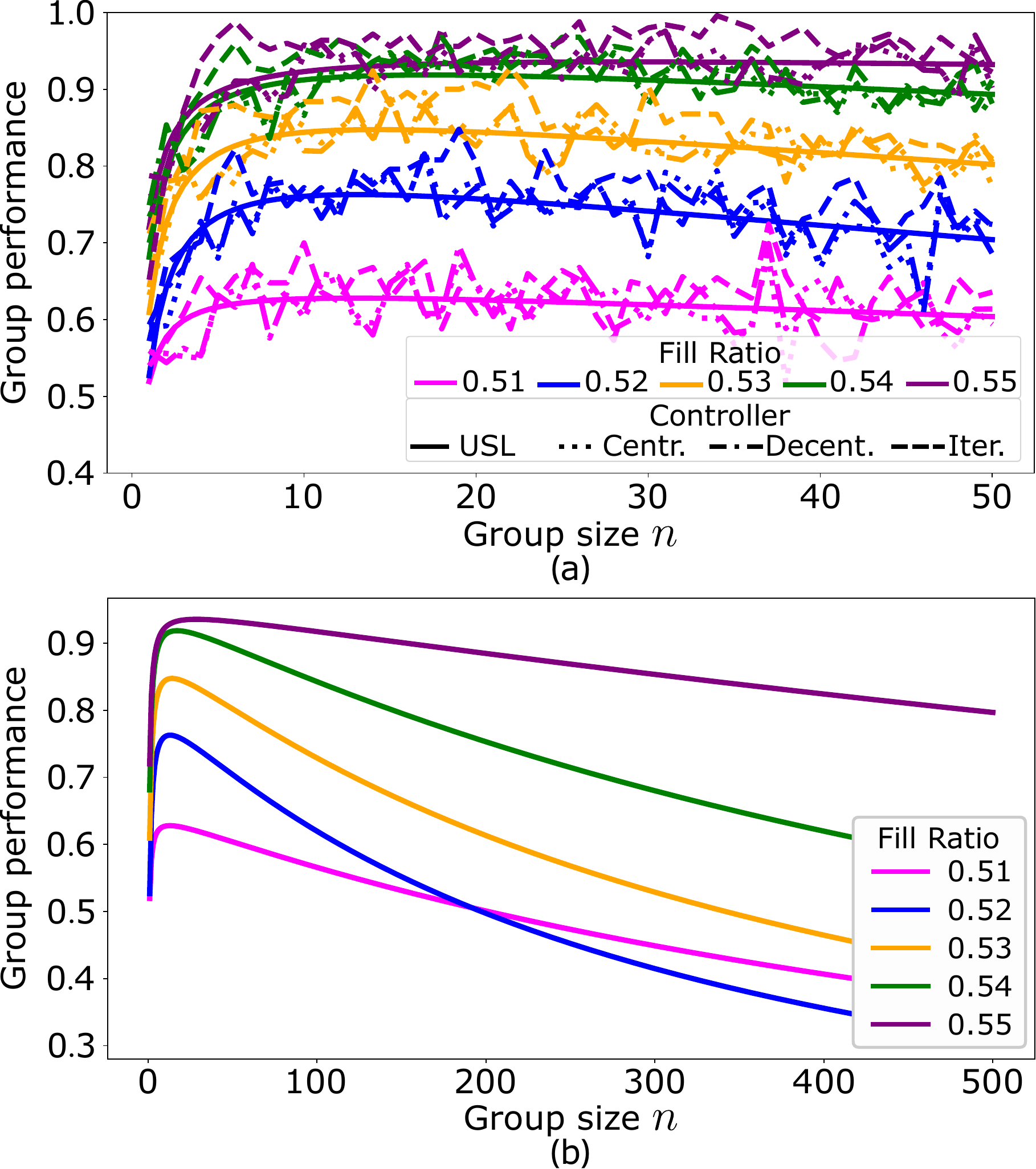}
  \caption{Scalability curves for multi-agent simulations with robot-to-robot physical interference. In~(a), the dashed/dotted lines show the group accuracy computed from 250 runs per condition in the checkerboard environment with fill ratio $f \in \{0.51, 0.52, 0.53, 0.54, 0.55\}$. The solid lines are the fitted USL curves with parameters obtained via nonlinear least squares fitting. In~(b), the fitted USL curves are extended to larger swarm sizes, using the estimated parameters in Table~\ref{tab:USLparams}.}
  \label{fig:interference}
\end{figure}

\section{Optimal Task Allocation Algorithm}
\label{sec:algorithm}
Next, we present our approach (Alg.~\ref{alg:task_alloc}) to allocate $N$~agents to $T$~tasks so that overall swarm performance (Eq.~\eqref{eqn:mult}) is maximized. 
Alg.~\ref{alg:task_alloc} is designed to find $\mathbf{N}^*$ from Eq.~\eqref{eqn:argmax}. Our underlying intuition is that of Cassey et al.~\cite{cassey2013adaptive}, the optimal allocation of resources is a function of the variability of the accumulation of evidence. Here, each task is characterized by its respective performance curves. 

Note that the total number of possible agent allocations grows exponentially with $N$ and $T$.
The number of possibilities of allocating~$N$ agents to $T$~tasks is given by the ordered partition function, that is, the number of possibilities of writing an integer~$N>0$ 
as the sum of $T>0$ positive integers in an ordered way. The ordered partition function for parameters~$N,T$ takes value ${N-1 \choose T-1}$, which is asymptotically $\Theta\left(\frac{2^N}{\sqrt{N}}\right)$ for $T=\lfloor \frac{N+1}{2}\rfloor$. 
Therefore, a brute-force approach to determine the optimal allocation is generally infeasible.

We begin by introducing the general version of our task allocation algorithm and prove that it computes the optimal solution under the assumption that each task performance function is concave. The computational complexity of the algorithm is analyzed in Sec.~\ref{sec:generalCase}, where we show that it runs in polynomial time.
We then show that the algorithm can be further simplified---reducing its complexity---by tailoring specific steps (marginal gain function) to each scalability function type: linear, saturating, and retrograde.

\subsection{The General Case} 
\label{sec:generalCase}

Depending on the performance function~$C(d_i, n_i)$ for task~$i$ with scalability parameter~$d_i$, we define the absolute gain function~$\Delta(d_i,n_i)$ and the marginal gain function~$\delta(d_i,n_i)$ as follows for~$n_i>0$:
\begin{align}
\Delta(d_i,n_i) &\coloneqq C(d_i,n_i+1) - C(d_i,n_i)\;, \\
\delta(d_i,n_i) &\coloneqq \frac{\Delta(d_i,n_i)}{C(d_i,n_i)}\;.
\end{align}
Given these definitions, we find how the addition of one agent to the $i$-th task changes the collective performance~$C$:
\begin{equation}\label{eq:overall-performance}
C(\mathbf{D},\mathbf{N}+e_i) = (1+\delta(d_i,n_i))C(\mathbf{D},\mathbf{N})\;.
\end{equation}

Alg.~\ref{alg:task_alloc} follows a simple principle: at each step, the next available agent is assigned to the task that yields the largest relative improvement in collective performance; concavity ensures that these relative gains decrease as more agents are assigned, guaranteeing the global optimality of this greedy selection.
Given the number of agents~$N$, the number of tasks~$T$, and the scalability parameters~$\mathbf{D}$, the algorithm first assigns one agent to each task (line~\ref{line:init-N}). Then it initializes an array of size~$T$ where each entry at respective position~$i$ represents the gain~$\delta$ achieved by assigning an additional agent to task~$i$. In each iteration (lines \ref{line:mainloop} to \ref{line:add2}), the task offering the highest gain is selected for the assignment of one new agent. In line~\ref{line:check}, we introduce a parameter $\varepsilon\ge0$ to define a threshold beyond which assigning an agent to a task is no longer considered to be beneficial. 
This threshold can be interpreted as the maintenance cost of a robot, meaning that if the gain from assigning an additional agent to a task is lower than $\varepsilon$, the cost of deployment outweighs any benefits of the operating agent. 
If all marginal gains are lower than this threshold, all remaining agents are placed in a `non-task'~$T+1$ that represents an idle pool of agents and the loop terminates. The algorithm continues looping until all agents are assigned.

\newcommand{\pr}[2]{C\left(#1,#2\right)}
\newcommand{\dfn}{=_{\rm def}}
\newcommand{\absgain}[2]{\Delta\left(#1,#2\right)}
\newcommand{\relgain}[2]{\delta\left(#1,#2\right)}
\newcommand{\success}[2]{C\left(#1,#2\right)}

\begin{algorithm}[t] 
\caption{Optimal task allocation algorithm}\label{alg:task_alloc}
\begin{flushleft}
\hspace*{\algorithmicindent}
\textbf{Input:} \\
\hspace*{\algorithmicindent} \hspace*{\algorithmicindent} number of agents $N$
\\
\hspace*{\algorithmicindent} \hspace*{\algorithmicindent}
number of tasks $T\le N$ \\
\hspace*{\algorithmicindent} \hspace*{\algorithmicindent} scalability parameters $\mathbf{D}=(d_1,\dots,d_T)$ \\
 \hspace*{\algorithmicindent} \textbf{Output:} best allocation $\mathbf{N}^*=(n_1,\dots,n_T,n_{T+1})$
 \end{flushleft}
\begin{algorithmic}[1]
\State $(n_1,\dots,n_T) \gets (1,\dots,1)$\label{line:init-N}
\ForAll{$i \in (1,\dots,T)$}\label{line:init_gain}
\State $marginal\_gain[i] \gets \delta(d_i,1)$
\EndFor \label{line:init-cjt-end}\label{line:end_init_gain}
\While {$\sum_{i=1}^{T}n_i < N $} \label{line:mainloop}
\State $m \gets {\rm arg} \max_i \ marginal\_gain[i]$ \label{line:argmax} 
\If { $marginal\_gain[m] \leq \varepsilon$ } \label{line:check}
\State$n_{T+1} \gets n_{T+1}+ (N - \sum_{i=1}^{T}n_i)$ \label{line:usl_type}
\State \textbf{break} \label{line:break}
\EndIf
\State $n_m \gets n_m+1$\label{line:add2}
\State  $ \scriptstyle marginal\_gain[m] \gets \delta(d_m,n_m)$ \label{line:update_gains}

\EndWhile
\end{algorithmic}
\end{algorithm}

\begin{theorem}
Alg.~\ref{alg:task_alloc} is optimal for all scalability functions of task performance $C(d,n)$, such that the marginal gain $\delta(d,n)$ is decreasing in $n$.
\end{theorem}

\begin{proof}
Let $\mathbf{N} = (n_1,\ldots,n_T)$ be the computed solution of Alg.~\ref{alg:task_alloc} and let $\mathbf{M} = (m_1,\ldots,m_T) \neq \mathbf{N}$ be some solution with $C(\mathbf{D},\mathbf{M}) \ge C(\mathbf{D},\mathbf{N})$.
Without loss of generality, assume $n_1 > m_1$ and $n_2 < m_2$.
Since Alg.~\ref{alg:task_alloc} selects in every step the task~$i$ maximizing the marginal gain and the gain function~$\delta(\cdot,\cdot)$ is monotonically decreasing in the second parameter, we have
$$
\delta(d_1,m_1) \ge \delta (d_1,n_1-1) \ge \delta(d_2,n_2)\ge\delta(d_2,m_2-1)\;.
$$
Now consider the solution $\mathbf{M}' \dfn (m_1+1, m_2-1, m_3,\ldots,m_T)$.
Eq.~\ref{eq:overall-performance} implies
$$
C(\mathbf{D},\mathbf{M}') = \frac{1+\delta(d_1,m_1)}{1+\delta(d_2,m_2-1)} C(\mathbf{D},\mathbf{M}) \ge C(\mathbf{D},\mathbf{M})\;.
$$
Hence, by repeating the above procedure, we can construct the solution $\mathbf{N}$ from $\mathbf{M}$ without decreasing the resulting collective performance.
This implies $C(\mathbf{D},\mathbf{M}) = C(\mathbf{D},\mathbf{N})$ and the theorem follows.
\end{proof}

We analyze the running time of Alg.~\ref{alg:task_alloc}.
In order to provide efficient access to the current maximal gain, the marginal gain values can be stored in a priority queue.
In this case, the initialization in lines \ref{line:init_gain} to \ref{line:end_init_gain} needs $\mathcal{O}(T\log(T))$.
The loop in lines \ref{line:mainloop} to \ref{line:add2} is executed $N-T$ times.
In each iteration, the algorithm needs to select and update the current maximal gain. 
We define $D(n_i)$ as the time needed to compute the $\delta(d_i,n_i)$. Since this time is at most $D(N)$, we can use $D(N)$ as a general upper bound in our complexity analysis, that is, $D(n_i)\le D(N)$. 
The extraction of the current maximum and insertion of the updated value is $\mathcal{O}(\log T)$.
In total, Alg.~\ref{alg:task_alloc} runs in time $\mathcal{O}(ND(N)\log(T)).$
In particular, if the marginal gain can be updated in constant time, the running time simplifies to $\mathcal{O}(N\log(T)).$

\subsection{Influence of Different Scalability Types}\label{sec:spec-case}

We analyze each of the three scalability functions---linear, saturating, and retrograde---in detail. For each, we examine how agents are optimally allocated and highlight the specific aspects or conditions required to guarantee the optimality of our algorithm.

\paragraph{Linear Scalability} Recall that the linear performance curve is defined as $C_\text{GL}(\lambda_i, n_i) = n_i - \lambda(n_i-1)$ where $\lambda_i \in [0,1]$ is the scalability parameter of task $i$.
The corresponding marginal gain is $\delta(\lambda_i, n_i) = \frac {1-\lambda_i} {n_i(1-\lambda_i) + \lambda_i}$.
Since this function is decreasing in~$n_i$, we can apply Alg.~\ref{alg:task_alloc} to compute the optimal task allocation.
Maximizing~$\delta(\lambda_i,n_i)$ is equivalent to selecting task~$i$ that minimizes the inverse $\frac1{\delta(\lambda_i,n_i)} = n_i + \frac{\lambda_i}{1-\lambda_i}$.

In the special case that $\lambda_i < \frac12$ for all tasks~$i$, we have $\frac{\lambda_i}{1-\lambda_i} < 1$, and the optimal solution corresponds to distributing agents evenly across tasks, with any remaining agents assigned first to those tasks with lower~$\lambda_i$.
In the general case, tasks with $\lambda_i \ge \frac12$---tasks with a large fraction of non-parallelizable work---receive fewer agents.
More precisely, consider two tasks with $\lambda_i > \lambda_j$.
The algorithm only selects task $i$ over task $j$ if $\delta(\lambda_i,n_i) > \delta(\lambda_j,n_j)$, which is equivalent to $n_j - n_i > \frac{\lambda_i-\lambda_j}{(1-\lambda_i)(1-\lambda_j)}$.
This means, if enough agents are available, the total difference in the number of assigned agents to both tasks is given by $\frac{\lambda_i-\lambda_j}{(1-\lambda_i)(1-\lambda_j)}$ (rounded up or down).
Therefore, the total difference in the number of assigned agents between tasks $i$ and $j$ is bounded by a constant determined solely by their parameters $\lambda_i$ and $\lambda_j$.
As the total number of agents $N$ increases, this fixed difference becomes relatively negligible, and the optimal allocation approaches a uniform distribution across tasks.

\paragraph{Saturating Scalability} We analyze the CJT curve~$C(p_i,n_i)$ defined in Eq.~\eqref{eq:cjt}, which models collective decision-making performance as a function of the number of agents $n_i$ and the individual accuracy probability~$p_i$. Notably, adding a second agent to make $n_i$ even does not improve performance over adding a first agent to make $n_i$ odd. 
One can formally verify that $C(p_i,n_i+1)=C(p_i,n_i)$ for odd~$n_i$.
This implies that the marginal gain $\delta(p_i,n_i) = 0$ for odd~$n_i$ and hence, the marginal gain function is not monotonically decreasing when agents are added one at a time.

To resolve this issue and preserve optimality, we modify the allocation process to assign agents in pairs.
This leads to a reformulated scalability function, defined as $\tilde{C}_\text{CJT}(p,k) =_{\rm def} C_\text{cond}(p,2k-1)$, where $(k-1)$ is the number of agent pairs assigned to a task.
When using this function, if $n-T$ is odd, one agent remains unallocated and can either be assigned arbitrarily or left unallocated, as it does not improve the collective performance defined in Eq.~\eqref{eq:overall-performance}.

Direct evaluation of the function $\tilde{C}_\text{CJT}$ is computationally costly and numerically unstable as $n_i$ increases.
To overcome this, we derive an incremental method, based on marginal gain, which requires constant computation times.
We assume that all individual accuracy probabilities~$p_i$ satisfy $1/2\le p_i\le 1$  in every task $i$.

By decomposing the binomial coefficients, the absolute gain function becomes:
\begin{equation}
\Delta(p,k)= {2k-1\choose k-1}(2p-1)(p(1-p))^k \label{eq:absolute-gain-explicit}\;.
\end{equation}
From this, we derive a recursive formula for the marginal gain $\delta(p,k)$, valid for $k>1$:
\begin{equation}\displaystyle\delta(p,k)=2(2-\frac{1}{k})p(1-p) \frac{\delta(p,k-1)}{1+\delta(p,k-1)}
\label{eq:marginal-gain}\;.
\end{equation}
This recurrence allows marginal gains to be computed efficiently, in constant time.

Finally, since $2(2-1/k) \le 4$ and $p(1-p) \le 1/4$, the marginal gain $\delta(p,k)$ is guaranteed to be monotonically decreasing in $k$. Therefore, Alg. \ref{alg:task_alloc} computes the optimal task allocation under saturating scalability.

\paragraph{Retrograde Scalability} 

Unlike the linear and saturating functions where collective performance increases monotonically with $n_i$, the USL function $C_\text{USL}(\alpha_i, \beta_i, x)$ of Eq.\,\eqref{eqn:usl} reaches a maximum at $n_i=\sqrt{(1-\alpha_i)/\beta_i}$. 
This introduces a qualitative difference: adding more agents not only yields diminishing returns but eventually becomes detrimental to performance, as the marginal gain $\delta<0$. 
In the discrete setting, the maximum is attained at either $n_i = \lfloor\sqrt{(1-\alpha_i)/\beta_i}\rfloor$ or $n_i = \lceil\sqrt{(1-\alpha_i)/\beta_i}\rceil$.
To achieve maximum performance, no additional agents are assigned to a task once the marginal gain $\delta(\alpha_i,\beta_i,n_i)$ becomes negative.

As a result, some agents may remain unallocated if all tasks have reached their respective performance peaks (i.e., saturation threshold). In Alg.~\ref{alg:task_alloc}, line~\ref{line:check} checks whether all marginal gains are non-positive under the current swarm allocation. If so, the remaining agents are allocated to the idle pool (task~$T+1$; see line \ref{line:usl_type}), finalizing the allocation process.

The marginal gain~$\delta(\alpha_i,\beta_i,n_i)$ is decreasing as long as the second derivative of $C_\text{USL}$ is non-positive.
If $\alpha_i\ge\beta_i$, the USL curve remains concave---without inflection points---within the interval $(0,\sqrt{(1-\alpha_i)/\beta_i})$, ensuring that Alg.~\ref{alg:task_alloc}  computes the optimal task allocation under retrograde scalability.

\section{Empirical Results}

We first verify the optimality of Alg.\,\ref{alg:task_alloc} by studying the allocation of robots across the collective decision-making tasks described in Sec.\,\ref{sec:case-study}. By considering only a limited number of tasks ($T=2$ and $T=3$) and agents ($N \le 150$), we can exhaustively explore the solution space and compute the optimal solution. A~small $T$ allows a simpler visualization of the possible allocations including the optimal one.
Then, in Sec.\,\ref{sec:res-many-tasks}, we demonstrate the possibility of scaling to a larger number of tasks (up to~$T=8$) and agents (up to $N=2000$). Although our algorithm can efficiently handle significantly more tasks ($T\gg8$), we limit the presentation to $T=8$ for clarity. In all scenarios, we set $\varepsilon=0$, meaning that there is no cost in deploying robots.

\subsection{Task Allocation Results with a Limited Number of Tasks}
\label{sec:res-task-alloc}

We test Alg.~\ref{alg:task_alloc} using results from our robot swarm simulations (Sec.~\ref{sec:majority_des}) to study the optimal allocation of $N=30$ robots among $T=2$ decision-making tasks exhibiting saturating scalability as shown in Fig.~\ref{fig:single-task-fill}. Each task~$i$ is characterized by its decision difficulty, that is, the difficulty for the robot to make a correct individual estimate. As shown in Sec.~\ref{sec:majority_des}, agents' individual decision accuracy depends on environmental factors, such as fill ratio~$f_i$ and feature correlation (e.g., patches of tiles with the same color). We define task~$T_1$ as either more difficult than task~$T_2$ or equally difficult, with difficulty determined by the fill ratios ($f_1\le f_2$), that is, agents have lower or equal individual decision accuracy~$p_1 \le p_2$. We use the individual accuracy value~$p$, obtained from robot simulations, to compute the group accuracy for task~$i$ using the saturating scalability curve $C_\text{CJT}(p_i,n_i)$ defined in Eq.\,(\ref{eq:cjt}), which closely fits the robot simulation data. 
The overall collective performance of the swarm of $N$~agents---see Eq.~\eqref{eqn:mult}---is computed by multiplying the group accuracies across all decision tasks. 

We present results for all possible allocations of the $N=30$ agents across the $T=2$ tasks (see Fig.~\ref{fig:mixed_task}). 
The collective accuracy shows an inverted U-shape that, with the peak position varying across task pairs---resulting in different maximum collective performances. 
Fig.~\ref{fig:mixed_task} shows that the allocation yielding the highest collective performance (circle) always matches the result of Alg.~\ref{alg:task_alloc} (cross). As expected, the optimal allocation consists shifts more agents towards the more difficult task, while in the symmetric case, the best solution is to divide the agents equally between the two tasks.

We extend the analysis to the allocation of $N=30$ or $N=150$ agents across $T=3$ decision-making tasks. Task~$T_1$ is always the most difficult, task~$T_2$ has medium difficulty, and task~$T_3$ is the simplest (fill ratios $f_1 < f_2 < f_3$). We show the results in Fig.~\ref{fig:ternary}(left) as color maps in ternary plots. 
Interestingly, the results for $N=30$ (left column) deviate from the trend observed in the case of $T = 2$~tasks. Contrary to intuitive expectations, the collective performance is not maximized when the allocation is biased towards the most difficult tasks. In the plots of Fig.~\ref{fig:ternary}(left), this would correspond to a region of high collective performance (dark red) near the bottom right corner associated with task $T_1$ (fill ratio $f_1=0.51$), which is not observed. 

In the results for swarm size $N=150$ (Fig.~\ref{fig:ternary}, right), the color maps exhibit a clear shift toward $T_1$, aligning more closely with our initial intuition. As swarm size increases, the optimal allocation becomes increasingly biased toward the more difficult tasks.
This shift results from the non-linearity of the CJT (Fig.~\ref{fig:single-task-fill}(b)) with respect to group size~$n$, which makes optimal allocation dependent on swarm size. These results highlight the importance of an efficient algorithm, as the optimal agent allocation must be recomputed for each swarm size, number of tasks, and the specific scalability functions associated with those tasks.

Fig.~\ref{fig:ternary} also shows that the allocation of our algorithm (cross symbol) always matches one of the allocations leading to maximum collective accuracy (circle symbols). For the $T=3$ tasks results, note that there are three optimal solutions and our algorithm always returns one of them. The presence of multiple equivalent optimal solutions is the consequence of the stepwise nature of CJT-based accuracy, where adding one agent to an odd group does not change the collective accuracy (see Supplementary Fig.~S4 for additional task allocation results). 

Fig.~\ref{fig:ternary_trend} clearly illustrates the size-dependent shift in allocation, with swarm sizes varying as 
$N\in\{5,10,15,\dots,1000\}$. The crosses in Fig.~\ref{fig:ternary_trend} indicate the optimal agent allocations computed by our algorithm, shown as proportions of the swarm ($\frac{n_i}{N} \in [0,1]$). As swarm size increases, agent allocation becomes increasingly biased in favor of the more difficult task ($T_1$ with $f_1=0.51$).  For tasks exhibiting saturating scalability, the optimal allocation may not be a fixed proportion of the swarm size.

\begin{figure}[t]
  \centering
  \includegraphics[width=1\linewidth]{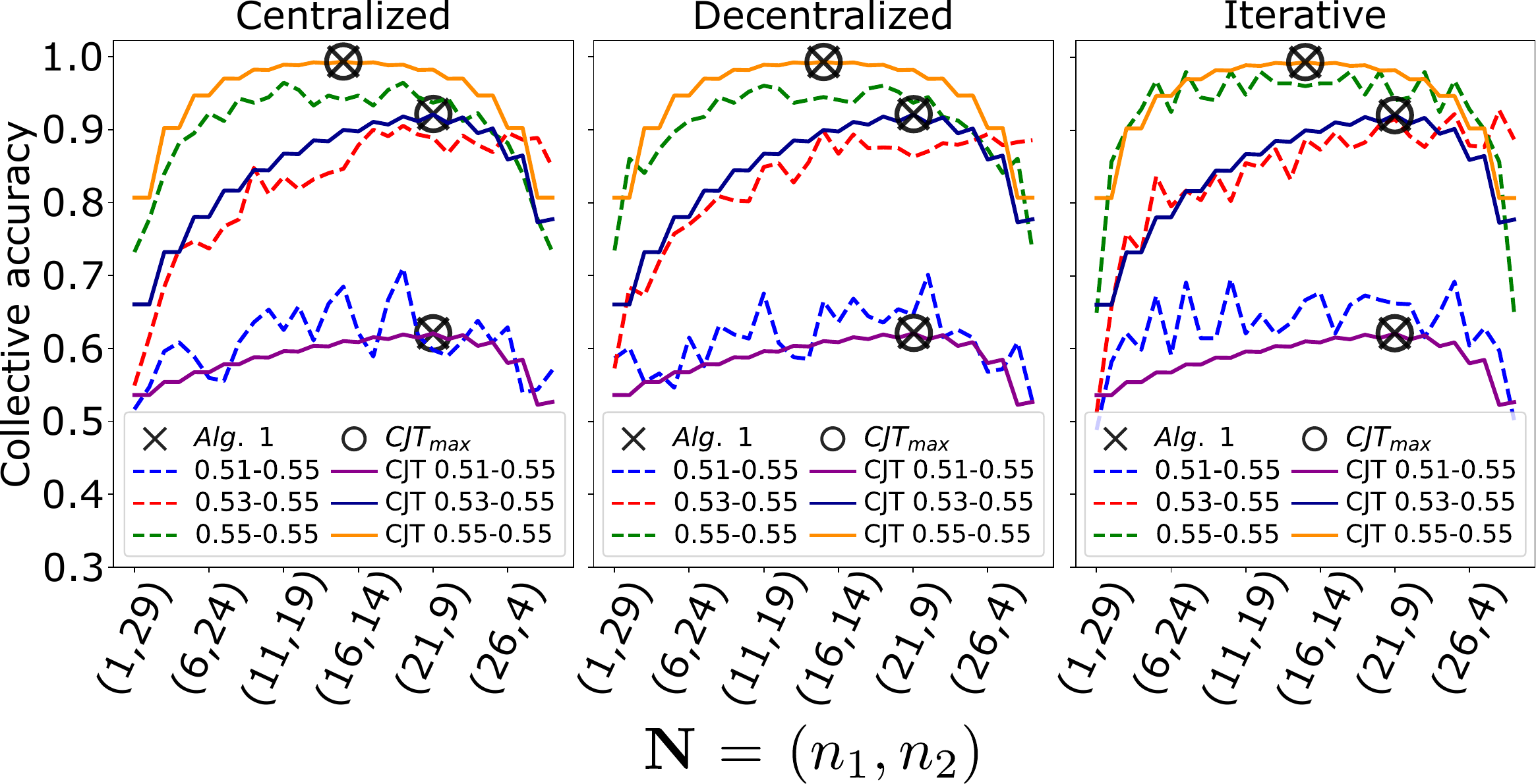}
  \caption{Collective performance for two-task environments ($T=2$) with fill ratios $(f_1, f_2) \in \{(0.51,0.55),(0.53,0.55), (0.55,0.55)\}$. We set $f_1\le f_2$; hence, task $T_1$, with fill ratio $f_1$, is the most difficult (or equally difficult) task compared to task $T_2$ with fill ratio $f_2$. On the x-axis, we vary the agent allocation $\mathbf{N}=(n_1,n_2)$ to the two tasks. Solid lines are obtained from the CJT (Eq.\,\eqref{eq:cjt}), and dashed lines show the results from robot simulations (250 repetitions per condition, data from Fig.~\ref{fig:single-task-fill}). The left, center, and right plots show the results for the centralized, decentralized, and iterative controllers, respectively. The cross marker indicates the agent allocation computed by Alg.~\ref{alg:task_alloc}, while the circle marker represents the maximum performance predicted by the CJT. In all scenarios, the cross and the circle coincide.}
  \label{fig:mixed_task}
\end{figure}

\begin{figure}[t]
  \centering
  \includegraphics[width=1\linewidth]{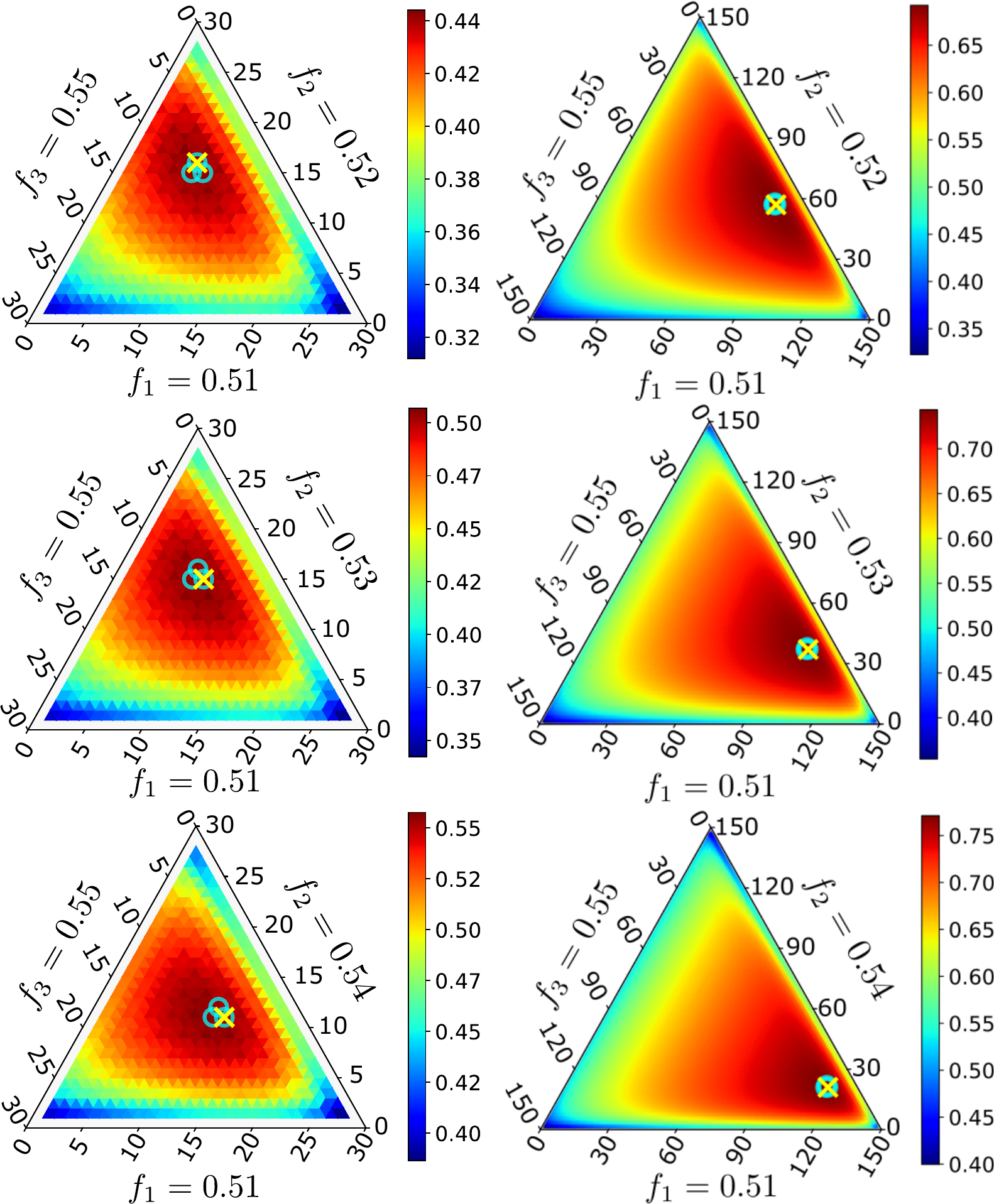}
  \caption{Three scenarios for the allocation of agents to $T=3$ collective decision-making tasks $(T_1,T_2,T_3)$ with fill ratios $(f_1,f_2,f_3)\in \{(0.51, 0.52,0.55), (0.51, 0.53,0.55), (0.51,0.54,0.55)\}$. We set $f_1 < f_2 < f_3$, hence task $T_1$ is always the most difficult, $T_2$ has intermediate difficulty, and $T_3$ is the easiest. The color map shows the collective performance computed using Eq.~\ref{eqn:mult} (see bar legends) for (left) small swarms with $N=30$ agents and (right) large swarms with $N=150$ agents. Blue circles mark the agent allocation yielding the highest performance, while the yellow cross indicates the allocation computed by our algorithm, which always coincides with one of the blue circles.}
  \label{fig:ternary}
\end{figure}

\begin{figure}[t]
  \centering
  \includegraphics[width=0.8\linewidth]{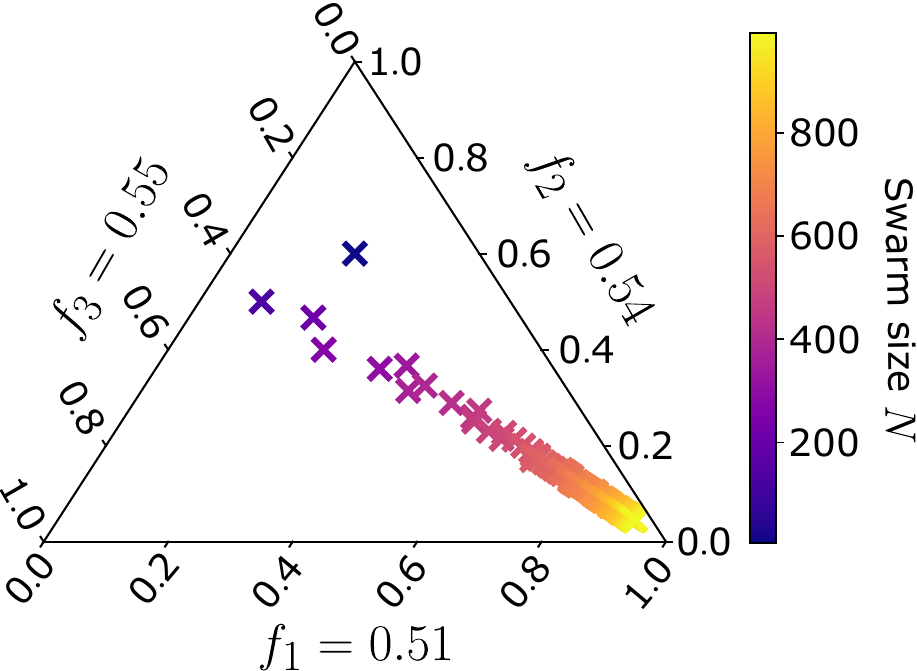}
  \caption{Allocation of agents to $T=3$ tasks  $(T_1,T_2,T_3)$ with fill ratios $(f_1,f_2,f_3)=(0.51,0.54,0.55)$. Crosses represent the optimal agent allocations computed by Alg.~\ref{alg:task_alloc} for swarm sizes $N \in \{5,10,15\dots,1000\}$, with each cross corresponding to a different swarm size. Axes of the ternary plot indicate the proportion of agents allocated to each respective task.}
  \label{fig:ternary_trend}
\end{figure}

\begin{figure*}[tb]
  \centering
  \includegraphics[width=0.8\linewidth]{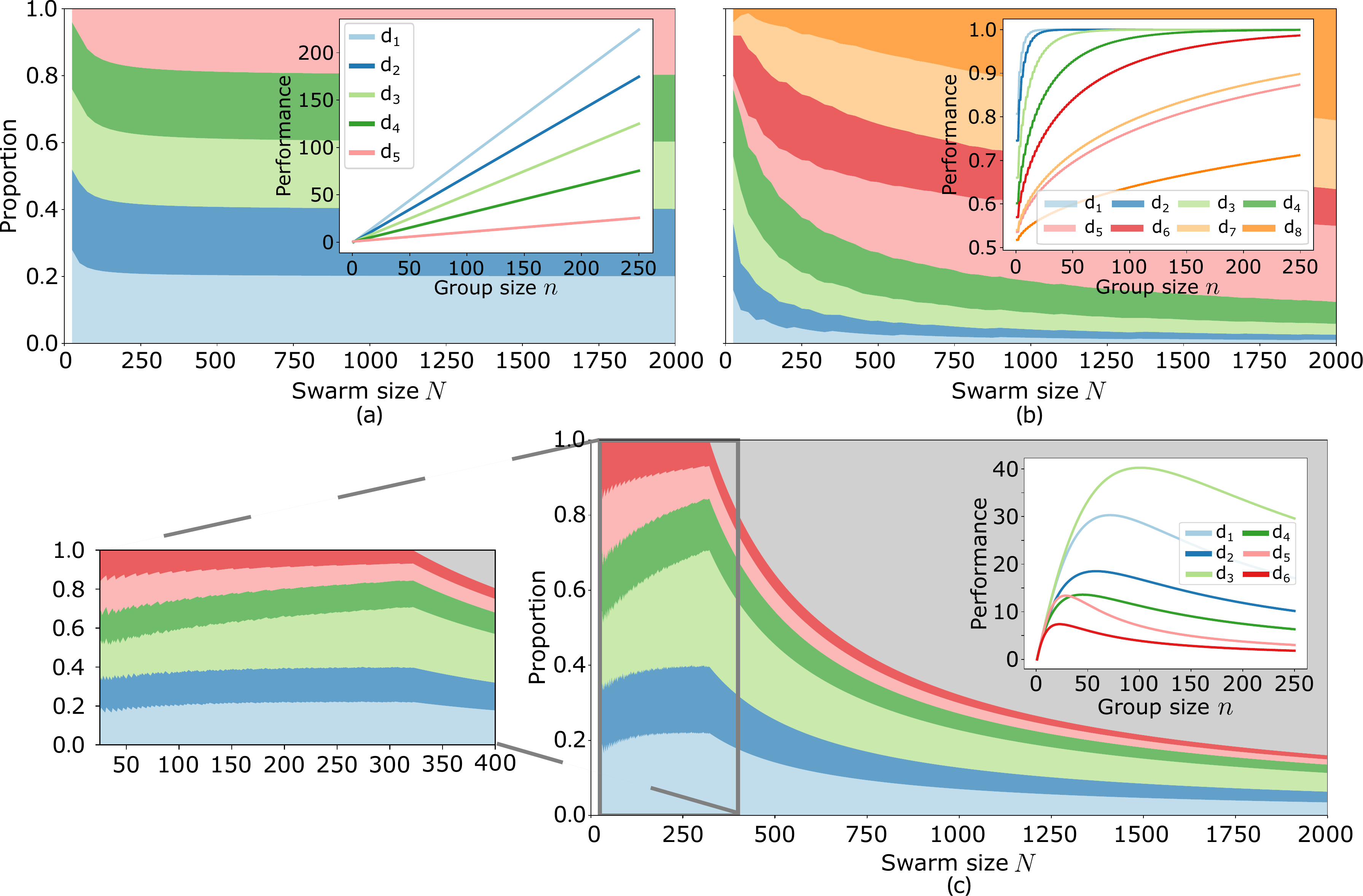}
  \caption{Stacked plots with linear, saturating and retrograde scalability performance curves. The stacked plot represents the proportion of agents assigned to each task, with colors corresponding to the tasks shown in the line plot for consistency. The light-gray color in the stacked plot represents the portion of unallocated agents. Line plots show the individual performance of each task with respect to swarm size. (a) The allocation of agents to linear $T=5$ tasks $(T_1,T_2,T_3,T_4,T_5)$ with $d_i = \lambda_i,\; \text{for } i \in \{1, 2, 3, 4, 5\}$ where $(d_1,d_2,d_3, d_4, d_5) \in \{0.1, 0.3, 0.5, 0.7, 0.9\}$. (b) The allocation of agents to saturating $T=8$ tasks $(T_1,T_2,T_3,T_4,T_5, T_6, T_7, T_8)$ with $d_i = p_i, \; \text{for } i \in \{1, 2, 3, 4, 5, 6, 7, 8\}$ where $(d_1,d_2,d_3, d_4, d_5, d_6, d_7, d_8) \in \{0.8069, 0.7454, 0.6603, 0.6017, 0.5361, 0.5698, 0.5402, 0.5177\}$ and line plots are presented in Fig.~\ref{fig:single-task-fill}(b) and Fig.~\ref{fig:single-task-dist}(b). (c)The allocation of agents to retrograde $T=6$ tasks $(T_1,T_2,T_3,T_4,T_5,T_6)$ with $d_i = (\alpha_i, \beta_i), \; \text{for } i \in \{1, 2, 3, 4, 5, 6\} \in \mathbb{Z}$ where $(d_1,d_2,d_3, d_4, d_5, d_6)\in \{(0.005,0.0002), (0.02,0.0003), (0.005,0.0001), (0.03,0.0005), \allowbreak (0.004,0.0013), (0.05,0.002)\}$.}
  \label{fig:stacked_all}
\end{figure*}

\subsection{Task Allocation Results with Many Tasks}
\label{sec:res-many-tasks}

One of the key strengths of our algorithm is its scalability, which allows it to handle scenarios involving multiple tasks and large numbers of agents efficiently. We analyze three scenarios (corresponding to linear, saturating, and retrograde scalability) and present the optimal allocation results for swarm sizes up to $N=2000$ agents, with $N \in \{25,26,27,\dots,2000\}$. 
Fig.~\ref{fig:stacked_all} shows the results as stacked plots, where each layer indicates the proportion of agents allocated to each task (i.e., $n_i/N$) as a function of swarm size $N$. Colors correspond to the scalability functions reported in the respective insets, except for the gray layer, which represents agents allocated to the idle pool (i.e., not assigned to any task). Idle agents only appear in panel (c) for tasks with retrograde scalability.

In Fig.~\ref{fig:stacked_all}(a), we analyze the linear case, considering $T = 5$ tasks with scalability parameters $(d_1,d_2,d_3, d_4, d_5) \in \{0.1, 0.3, 0.5, 0.7, 0.9\}$. Because $d_1$, $d_2$, and $d_3$ are smaller than $0.5$, the optimal allocation consists of an equal distribution of agents among $T_1$, $T_2$, and $T_3$. In contrast, since $d_4$ and $d_5$ are larger than $0.5$, fewer agents are allocated to $T_4$ and $T_5$ than the other three tasks (see also discussion in Sec.~\ref{sec:spec-case}). This difference in optimal group sizes among tasks is relatively small in the number of allocated agents (about 6 to 8 agents). Because the stacked plot reports proportions of agents, this small difference is graphically accentuated for small swarm sizes. For example, for $N=25$, the difference between $n_1$ and $n_5$ is 6 agents, corresponding to $24\%$ of the swarm; instead, for $N=2000$, the difference $n_1-n_5$ is 8 agents, corresponding to $0.4\%$ of $N$ only. However, except for small swarms, the optimal group sizes $n_i$ are approximately a constant proportion $n_i=1/T$.

In Fig.~\ref{fig:stacked_all}(b), we analyze the saturating case (i.e., diminishing returns), considering $T=8$ tasks with scalability parameters~$d_i$ measured from our robot simulations and reported in Tab.~\ref{tab:individualAcc} as individual accuracies~$p_i$. Similar to the results of Sec.~\ref{sec:res-task-alloc}, we observe a size-dependent task allocation, meaning that the optimal allocation is not a constant proportion of the swarm size but changes as~$N$ increases. These results are a consequence of different values of~$p_i$ (see inset), leading to marginal gains decreasing differently across tasks. Fig.\,\ref{fig:stacked_all}(b) illustrates these results as stacks changing height (y-axis) as a function of $N$ (x-axis). When allocating small swarms, the largest groups are assigned to tasks with the highest~$p_i$ values (i.e., the easiest tasks) because they provide the highest marginal gain due to a steeper increase in performance compared to more difficult tasks with lower $p_i$. In contrast, when allocating large swarms, the more difficult tasks receive the largest proportion of agents because the easiest tasks have already approached the maximum performance (e.g., blue tasks in the inset) and provide minimal marginal gain compared to the harder tasks.

In Fig.~\ref{fig:stacked_all}(c), we analyze the retrograde case considering $T=6$ tasks illustrated in the inset. A~distinctive characteristic of retrograde tasks is the presence of a peak performance at a given group size~$n_i$. Beyond this point, assigning more agents becomes detrimental and reduces the task performance, corresponding to a negative marginal gain. 
When all tasks have negative marginal gain, $\forall i, \delta(d_i,n_i) < 0$, or, more precisely, smaller than the robot deployment cost $\varepsilon$ (assumed to be $\varepsilon=0$ for simplicity), our algorithm stops allocating agents to tasks and all unallocated agents are assigned to the idle pool (gray stack in Fig.~\ref{fig:stacked_all}(c)). In the tested scenario, that point is reached at swarm size~$N=322$, beyond which the task allocations across the $T$ tasks remain unchanged, and any additional agents are assigned to $T_{T+1}$ only.

\section{Conclusion}

Computing optimal task allocations for multi-agent systems is essential for maximizing overall performance. However, it can be challenging and computationally expensive, especially when dealing with large numbers of agents and tasks. In large-scale systems, evaluating all possible allocations and selecting the best one becomes computationally infeasible due to combinatorial explosion. We propose a polynomial-time algorithm for optimally allocating any number of agents to any number of tasks, applicable to any concave scalability function describing how performance depends on group size.

In our study, we model tasks by taking inspiration from the performance scaling principles of parallel computing systems. We run a series of analyses both to showcase the optimality of our algorithm and to study how optimal task allocations change for different task scalability functions and group sizes. We find that in linear and retrograde scalability, task allocation strategies are generally intuitive: linear tasks lead to an approximately even distribution of agents, while retrograde tasks require allocating agents up to their performance peak, leaving any excess agents unassigned. However, saturating performance functions---and retrograde tasks before reaching their peak---present unique challenges. The diminishing returns characteristic of these curves require a more refined allocation strategy.

In addition to multi-agent simulations, we also run a series of experiments with simulated robot swarms performing a benchmark case study in swarm robotics: collective perception of environmental features~\cite{valentini2016collective, ebert2020bayes, zakir2022robot}. This is a type of collective decision-making scenario where agents select the predominant environmental feature through majority decisions. We model the different task difficulties by changing either the feature's frequency (i.e., fill ratio) or the feature's spatial correlation~\cite{bartashevich2021multi} (i.e., sizes of unicolor patches). These collective decision-making tasks scale with swarm size with a saturating curve---following the Condorcet’s Jury Theorem (CJT)---when there is no interference among robots. We also show that when we include body collisions, physical interference among robots changes the scalability function to retrograde. The swarm robotics simulations showcase the relevance of the considered scalability functions for distributed systems and the importance of task allocation for improving the performance of large-scale systems.

We naively assumed that it would always be optimal to bias the allocation of agents in favor of the most difficult tasks. Our initial intuition was based on results from social science on human behavior~\cite{cassey2013adaptive}, where the best allocation of resources (time) to each task is proportional to the task difficulty; however, our analysis contradicts these earlier findings. On the contrary, our theoretical and computational analyses indicate that this solution is only optimal when resources are abundant (i.e., in large swarms). In contrast, in small swarms, assigning more agents to the most difficult task can be suboptimal. These results highlight the importance of an efficient algorithm for task allocation because the optimal solution is not a linear rescaling of smaller-scale systems. Instead, the best allocation should be recomputed as system size, number of tasks, or task difficulties change.

Allocating agents to different tasks is relevant and important in both engineering and biology. For example, in swarm robotics, system efficiency can be improved by distributing the workload among the robots. Assuming different tasks are performed using different types of sensors, task allocation can also be exploited for cost-effective design of robot swarms, where sensors are allocated as a function of each task's requirements. 
Modeling this problem could also help in understanding biological systems, for example, the composition of mixed-species groups~\cite{goodale2017mixed}, in which animals of different species cooperate by contributing complementary sensing capabilities to the group (e.g., for predator detection).

The current study considers an offline allocation setting with homogeneous agents, centralized computation, independent tasks, and known scalability functions. While these assumptions enable optimality guarantees and polynomial-time complexity, they abstract away some challenges inherent to real-world swarm systems, including agent heterogeneity, stochastic task dynamics, partial observability, and communication constraints. A~natural extension is to replace the predefined scalability functions with online performance models learned from locally observed task performance, enabling adaptive task allocation under uncertainty.

Future research can also extend the present study to include heterogeneity among agents and tasks. Our algorithm assumes homogeneous agents, where any agent contributes equally to a task, and it could be potentially generalized to optimally allocate swarms of heterogeneous agents---i.e., certain robots have better capabilities than others in performing the task (e.g., sampling the environment and making the correct decision~\cite{hoeffding1956distribution}). Our algorithm can also be extended---by merely changing the marginal gain calculation step---to handle concurrently tasks with different types of scalability functions. These future extensions could broaden the applicability of our method to real-world applications with heterogeneous tasks and agents. Our framework applies to any concave scalability function. While we mainly focus in linear, saturating and retrograde functions, other types of concave functions (e.g., logistic growth) would lead to similar allocation patterns. The critical property is concavity, which ensures diminishing marginal gains and keeps the optimization computationally feasible. Non-concave functions could exhibit multiple local optima and would require different algorithmic tools, which we leave for future work.

We hope that our work not only provides a foundation for optimal task allocation in scalable multi-agent systems but also helps improve approaches for task allocation in swarm robotics and collective decision-making. There are many open questions about dynamically changing scalability functions, online task allocation, group sizes changing at runtime (e.g., due to broken robots), and a fully decentralized approach. Our presented insights may inspire cross-disciplinary research, bridging robotics, distributed computing, social science, and biological systems. We believe our approach is just the beginning of research on modern techniques for complex task allocation problems for the next generation of intelligent multi-robot systems, unlocking new possibilities for large-scale real-world deployment of robot systems.

\section*{Acknowledgments}
This work has been partially supported by the DFG under Germany's Excellence Strategy, EXC 2117, 422037984 (H.H. and A.R.) and by the Hector Foundation~II.
We thank Samer Al-Magazachi for kindly providing us with the source code of his Python-based robot swarm simulator `Swarmy.'

\bibliographystyle{ieeetr} 
\bibliography{references}

@article{barnard1981producers,
  title={Producers and scroungers: a general model and its application to captive flocks of house sparrows},
  author={Barnard, Christopher J and Sibly, Richard M},
  journal={Animal Behaviour},
  volume={29},
  number={2},
  pages={543--550},
  year={1981}
}

@article{cantrell2007evolution,
  title={Evolution of dispersal and ideal free distribution in populations with logistic growth rates},
  author={Cantrell, Robert S and Cosner, Chris and Lou, Yuan},
  journal={Mathematical Biosciences and Engineering},
  volume={4},
  number={2},
  pages={317--330},
  year={2007}
}

@article{monderer1996potential,
  title={Potential games},
  author={Monderer, Dov and Shapley, Lloyd S},
  journal={Games and Economic Behavior},
  volume={14},
  number={1},
  pages={124--143},
  year={1996}
}

@inproceedings{valentini2016collective,
  title={Collective perception of environmental features in a robot swarm},
  author={Valentini, Gabriele and Brambilla, Davide and Hamann, Heiko and Dorigo, Marco},
  booktitle={Proceedings of the International Conference on Swarm Intelligence (ANTS)},
  pages={65--76},
  year={2016},
  publisher={Springer}
}

@Article{  gunther15,
  author = 	 {Neil J. Gunther and P. Puglia and K. Tomasette},
  title = 	 {Hadoop Super-linear Scalability: The perpetual motion of parallel
performance},
  journal = 	 {ACM Queue},
  year = 	 {2015},
  OPTkey = 	 {},
  volume = 	 {13},
  number = 	 {5},
  pages = 	 {46-55},
  OPTmonth = 	 {},
  OPTnote = 	 {},
  OPTannote = 	 {}
}

@inproceedings{gunther93,
  author = {Gunther, Neil J.},
  title = {A Simple Capacity Model of Massively Parallel Transaction Systems},
  booktitle = {Proceedings of the Computer Measurement Group (CMG) National Conference},
  pages = {1035--1044},
  year = {1993}
}

@INPROCEEDINGS{kuckling24,
  author={Kuckling, Jonas and Luckey, Robin and Avrutin, Viktor and Vardy, Andrew and Reina, Andreagiovanni and Hamann, Heiko},
  booktitle={IEEE International Conference on Robotics and Automation (ICRA)}, 
  title={Do We Run Large-scale Multi-Robot Systems on the Edge? More Evidence for Two-Phase Performance in System Size Scaling}, 
  year={2024},
  volume={},
  number={},
  pages={4562-4568},
  publisher={IEEE},
  doi={10.1109/ICRA57147.2024.10610771}
}

@article{gustafson1988,
  author = {Gustafson, John L.},
  title = {Reevaluating {A}mdahl's Law},
  journal = {Communications of the ACM},
  volume = {31},
  number = {5},
  pages = {532--533},
  year = {1988},
  publisher = {ACM}
}

@book{kroese2013handbook,
  title={Handbook of Monte Carlo methods},
  author={Kroese, Dirk P and Taimre, Thomas and Botev, Zdravko I},
  year={2013},
  publisher={John Wiley \& Sons},
  address = {Hoboken, NJ, USA}
}

@inproceedings{shan2020,
  title={Collective Decision Making in Swarm Robotics with Distributed Bayesian Hypothesis Testing},
  author={Shan, Qihao and Mostaghim, Sanaz},
  booktitle={Proceedings of the International Conference on Swarm Intelligence (ANTS)},
  pages={55--67},
  year={2020},
  publisher={Springer}
}

@article{sasaki2013ants,
  title={Ants learn to rely on more informative attributes during decision-making},
  author={Sasaki, Takao and Pratt, Stephen C},
  journal={Biology Letters},
  volume={9},
  number={6},
  pages={20130667},
  year={2013},
  publisher={The Royal Society}
}

@article{kao2014decision,
  title={Decision accuracy in complex environments is often maximized by small group sizes},
  author={Kao, Albert B and Couzin, Iain D},
  journal={Proceedings of the Royal Society B: Biological Sciences},
  volume={281},
  number={1784},
  pages={20133305},
  year={2014},
  publisher={The Royal Society}
}

@article{khaluf2019collective,
  title={Collective sampling of environmental features under limited sampling budget},
  author={Khaluf, Yara and Simoens, Pieter},
  journal={Journal of Computational Science},
  volume={31},
  pages={95--110},
  year={2019},
  publisher={Elsevier}
}

@article{cassey2013adaptive,
  title={Adaptive sampling of information in perceptual decision-making},
  author={Cassey, Thomas C and Evens, David R and Bogacz, Rafal and Marshall, J. A. R. and Ludwig, Casimir J. H.},
  journal={PLOS ONE},
  volume={8},
  number={11},
  pages={e78993},
  year={2013},
  publisher={Public Library of Science}
}

@article{fretwell1969territorial,
  title={On territorial behavior and other factors influencing habitat distribution in birds: II. Sex ratio variation in the Dickcissel (Spiza americana Gmel)},
  author={Fretwell, Stephen Dewitt and Calver, James Stevan},
  journal={Acta Biotheoretica},
  volume={19},
  number={1},
  pages={37--44},
  year={1969},
  publisher={Springer}
}

@article{kacelnik1992ideal,
  title={The ideal free distribution and predator-prey populations},
  author={Kacelnik, Alejandro and Krebs, John R and Bernstein, Carlos},
  journal={Trends in Ecology \& Evolution},
  volume={7},
  number={2},
  pages={50--55},
  year={1992},
  publisher={Elsevier Current Trends}
}

@article{quijano2007ideal,
  title={The ideal free distribution: Theory and engineering application},
  author={Quijano, Nicanor and Passino, Kevin M},
  journal={IEEE Transactions on Systems, Man, and Cybernetics, Part B (Cybernetics)},
  volume={37},
  number={1},
  pages={154--165},
  year={2007},
  publisher={IEEE}
}

@article{ gerkey04,
author = {Brian P. Gerkey and Maja J. Matari{\'c}},
title ={A Formal Analysis and Taxonomy of Task Allocation in Multi-Robot Systems},
journal = {The International Journal of Robotics Research},
volume = {23},
number = {9},
pages = {939-954},
year = {2004},
}

@article{lerman2006analysis,
  title={Analysis of dynamic task allocation in multi-robot systems},
  author={Lerman, Kristina and Jones, Chris and Galstyan, Aram and Matari{\'c}, Maja J},
  journal={The International Journal of Robotics Research},
  volume={25},
  number={3},
  pages={225--241},
  year={2006},
  publisher={SAGE Publications}
}

@article{berman2009optimized,
  title={Optimized stochastic policies for task allocation in swarms of robots},
  author={Berman, Spring and Hal{\'a}sz, Ad{\'a}m and Hsieh, M Ani and Kumar, Vijay},
  journal={IEEE Transactions on Robotics},
  volume={25},
  number={4},
  pages={927--937},
  year={2009},
  publisher={IEEE}
}

@article{kanakia2016modeling,
  title={Modeling multi-robot task allocation with limited information as global game},
  author={Kanakia, Anshul and Touri, Behrouz and Correll, Nikolaus},
  journal={Swarm Intelligence},
  volume={10},
  pages={147--160},
  year={2016},
  publisher={Springer}
}

@inproceedings{fleming2019recruitment,
  title={Recruitment-Based Robotic Colony Allocation},
  author={Fleming, Chloe and Adams, Julie A},
  booktitle={Distributed Autonomous Robotic Systems: The 14th International Symposium},
  pages={79--94},
  year={2019},
  publisher={Springer},
  address = {Berlin, Germany}
}

@article{jang2018anonymous,
  title={Anonymous hedonic game for task allocation in a large-scale multiple agent system},
  author={Jang, Inmo and Shin, Hyo-Sang and Tsourdos, Antonios},
  journal={IEEE Transactions on Robotics},
  volume={34},
  number={6},
  pages={1534--1548},
  year={2018},
  publisher={IEEE}
}

@article{bogomolnaia2002stability,
  title={The stability of hedonic coalition structures},
  author={Bogomolnaia, Anna and Jackson, Matthew O},
  journal={Games and Economic Behavior},
  volume={38},
  number={2},
  pages={201--230},
  year={2002},
  publisher={Elsevier}
}

@article{paulk2014selective,
  title={Selective attention in the honeybee optic lobes precedes behavioral choices},
  author={Paulk, Angelique C and Stacey, Jacqueline A and Pearson, Thomas WJ and Taylor, Gavin J and Moore, Richard JD and Srinivasan, Mandyam V and Van Swinderen, Bruno},
  journal={Proceedings of the National Academy of Sciences},
  volume={111},
  number={13},
  pages={5006--5011},
  year={2014},
  publisher={National Acad Sciences}
}

@article{tipper1988negative,
  title={Negative priming between response modalities: Evidence for the central locus of inhibition in selective attention},
  author={Tipper, Steven P and MacQueen, Glenda M and Brehaut, Jamie C},
  journal={Perception \& Psychophysics},
  volume={43},
  number={1},
  pages={45--52},
  year={1988},
  publisher={Springer}
}

@article{glynn1989importance,
  title={Importance sampling for stochastic simulations},
  author={Glynn, Peter W and Iglehart, Donald L},
  journal={Management Science},
  volume={35},
  number={11},
  pages={1367--1392},
  year={1989},
  publisher={INFORMS}
}

@article{neal2001annealed,
  title={Annealed importance sampling},
  author={Neal, Radford M},
  journal={Statistics and computing},
  volume={11},
  pages={125--139},
  year={2001},
  publisher={Springer}
}

@inproceedings{ebert2018multi,
  title={Multi-feature collective decision making in robot swarms},
  author={Ebert, Julia T and Gauci, Melvin and Nagpal, Radhika},
  booktitle={Proceedings of the 17th International Conference on Autonomous Agents and MultiAgent Systems (AAMAS)},
  pages={1711--1719},
  year={2018}
}

@book{goodale2017mixed,
  title={Mixed-species groups of animals: behavior, community structure, and conservation},
  author={Goodale, Eben and Beauchamp, Guy and Ruxton, Graeme D},
  year={2017},
  publisher={Academic Press},
address = {San Diego, CA, USA}

}

@book{condorcet1785essay,
  author = {Condorcet, Marquis de},
  title = {Essai sur l'application de l'analyse \`a la probabilit\'e des d\'ecisions rendues \`a la pluralit\'e des voix},
  year = {1785},
  publisher = {Imprimerie Royale},
  address = {Paris}
}

@article{hoeffding1956distribution,
  title={On the distribution of the number of successes in independent trials},
  author={Hoeffding, Wassily},
  journal={The Annals of Mathematical Statistics},
  pages={713--721},
  year={1956},
  publisher={JSTOR}
}

@book{hamann2018swarm,
  title={Swarm robotics: A formal approach},
  author={Hamann, Heiko},
  volume={221},
  year={2018},
  publisher={Springer},
  address = {Berlin, Germany}

}

@article{bartashevich2021multi,
  title={Multi-featured collective perception with evidence theory: tackling spatial correlations},
  author={Bartashevich, Palina and Mostaghim, Sanaz},
  journal={Swarm Intelligence},
  volume={15},
  number={1},
  pages={83--110},
  year={2021},
  publisher={Springer}
}

@inproceedings{zakir2022robot,
  title={Robot swarms break decision deadlocks in collective perception through cross-inhibition},
  author={Zakir, Raina and Dorigo, Marco and Reina, Andreagiovanni},
  booktitle={Proceedings of the International Conference on Swarm Intelligence (ANTS)},
  pages={209--221},
  year={2022},
  publisher={Springer}
}

@inproceedings{ebert2020bayes,
  title={Bayes Bots: Collective Bayesian decision-making in decentralized robot swarms},
  author={Ebert, Julia T and Gauci, Melvin and Mallmann-Trenn, Frederik and Nagpal, Radhika},
  booktitle={2020 IEEE international conference on robotics and automation (ICRA)},
  pages={7186--7192},
  year={2020},
  publisher={IEEE}
}

@article{king2007use,
  title={When to use social information: the advantage of large group size in individual decision making},
  author={King, Andrew J and Cowlishaw, Guy},
  journal={Biology Letters},
  volume={3},
  number={2},
  pages={137--139},
  year={2007},
  publisher={The Royal Society London}
}

@article{hamann2021scalability,
  title={Scalability in computing and robotics},
  author={Hamann, Heiko and Reina, Andreagiovanni},
  journal={IEEE Transactions on Computers},
  volume={71},
  number={6},
  pages={1453--1465},
  year={2021},
  publisher={IEEE}
}

@article{prorok2017impact,
  title={The impact of diversity on optimal control policies for heterogeneous robot swarms},
  author={Prorok, Amanda and Hsieh, M Ani and Kumar, Vijay},
  journal={IEEE Transactions on Robotics},
  volume={33},
  number={2},
  pages={346--358},
  year={2017},
  publisher={IEEE}
}

@inproceedings{hamann2020guerrilla,
  title={Guerrilla performance analysis for robot swarms: Degrees of collaboration and chains of interference events},
  author={Hamann, Heiko and Aust, Till and Reina, Andreagiovanni},
  booktitle={Proceedings of the International Conference on Swarm Intelligence (ANTS)},
  pages={134--147},
  year={2020},
  publisher={Springer}
}

@inproceedings{balch2000reward,
  title={Reward and Diversity in Multirobot Foraging},
  author={Balch, Tucker},
  booktitle={Proceedings of the AAAI 2000 Spring Symposium on Search Techniques for Problem Solving under Uncertainty and Incomplete Information},
  year={2000},
  organization={AAAI},
  url={https://citeseerx.ist.psu.edu/document?repid=rep1&type=pdf&doi=074dfa60edf5d724cdc91320c1d67bf8af77e756}
}

@incollection{ROBINSON2009297,
title = {Chapter 77 - Division of Labor in Insect Societies},
editor = {Vincent H. Resh and Ring T. Cardé},
booktitle = {Encyclopedia of Insects (Second Edition)},
publisher = {Academic Press},
edition = {Second Edition},
address = {San Diego},
pages = {297-299},
year = {2009},
isbn = {978-0-12-374144-8},
author = {Gene E. Robinson}
}

@inproceedings{zakir2024miscommunication,
  title={Miscommunication between robots can improve group accuracy in best-of-n decision-making},
  author={Zakir, Raina and Dorigo, Marco and Reina, Andreagiovanni},
  booktitle={2024 IEEE/RSJ International Conference on Intelligent Robots and Systems (IROS)},
  pages={9014--9021},
  year={2024},
  organization={IEEE}
}

@inproceedings{aziz2021multi,
  title={Multi-Robot Task Allocation—Complexity and Approximation},
  author={Aziz, Haris and Chan, Hau and Cseh, {\'A}gnes and Li, Bo and Ramezani, Fahimeh and Wang, Chenhao},
  booktitle={Proceedings of the 20th International Conference on Autonomous Agents and MultiAgent Systems (AAMAS)},
  pages={133--141},
  year={2021},
  organization={IFAAMAS}
}

@article{aziz2022task,
  title={Task allocation using a team of robots},
  author={Aziz, Haris and Pal, Arindam and Pourmiri, Ali and Ramezani, Fahimeh and Sims, Brendan},
  journal={Current Robotics Reports},
  volume={3},
  number={4},
  pages={227--238},
  year={2022},
  publisher={Springer}
}

@inproceedings{bachrach2008distributed,
  title={Distributed multiagent resource allocation in diminishing marginal return domains},
  author={Bachrach, Yoram and Rosenschein, Jeffrey S.},
  booktitle={Proceedings of the 7th International Joint Conference on Autonomous Agents and Multiagent Systems (AAMAS)},
  pages={1103--1110},
  year={2008}
}

@inproceedings{soma2014optimal,
  title={Optimal budget allocation: Theoretical guarantee and efficient algorithm},
  author={Soma, Tasuku and Kakimura, Naonori and Inaba, Kazuhiro and Kawarabayashi, Ken-ichi},
  booktitle={International Conference on Machine Learning},
  pages={351--359},
  year={2014},
  organization={PMLR}
}

@inproceedings{bartashevich2019benchmarking,
  title={Benchmarking collective perception: New task difficulty metrics for collective decision-making},
  author={Bartashevich, Palina and Mostaghim, Sanaz},
  booktitle={EPIA Conference on Artificial Intelligence},
  pages={699--711},
  year={2019},
  organization={Springer}
}

\end{document}